\documentclass[twoside,leqno,twocolumn]{article}

\usepackage[letterpaper]{geometry}

\usepackage{ltexpprt}
\usepackage{hyperref}
\usepackage{amssymb}
\usepackage{amsmath}
\usepackage{amsfonts}

\usepackage{booktabs}
\usepackage{float}

\usepackage{algorithm}
\usepackage{algorithmic}

\usepackage{subcaption}
\newcommand{\argmin}{\operatornamewithlimits{argmin }}

\newcommand{\minimize}{\operatornamewithlimits{min }}

\newcommand{\subjto}{\operatornamewithlimits{\text{subject to }}}
\newcommand{\ones}{\mathbf 1}
\newcommand{\reals}{{\mathbb R}}

\newcommand{\prox}{\operatorname{prox}}
\author{}

\usepackage{mathtools}

\usepackage{tikz}
\usetikzlibrary{arrows.meta, positioning, shapes.geometric}

\graphicspath{{./Figures/}}
\begin{document}

\title{Rank Aggregation with Confidence for Large Scale Comparison Data}
\author{Filipa Valdeira\thanks{NOVA Laboratory for Computer
Science and Informatics (NOVA LINCS), NOVA School of Science and Technology (NOVA FCT). E-mails: \{f.valdeira, claudia.soares\}@fct.unl.pt} \thanks{Center for Mathematics and Applications (NOVA Math), NOVA School of Science and Technology (NOVA FCT). }
\and Cláudia Soares\footnotemark[1]}

\date{}

\maketitle

\begin{abstract}
In this work, we leverage a generative data model considering comparison noise to develop a fast, precise, and informative ranking algorithm from pairwise comparisons that produces a measure of confidence on each comparison. The problem of ranking a large number of items from noisy and sparse pairwise comparison data arises in diverse applications, like ranking players in online games, document retrieval or ranking human perceptions. Although different algorithms are available, we need fast, large-scale algorithms whose accuracy degrades gracefully when the number of comparisons is too small. Fitting our proposed model entails solving a non-convex optimization problem, which we tightly approximate by a sum of quasi-convex functions and a regularization term. Resorting to an iterative reweighted minimization and the Primal-Dual Hybrid Gradient method, we obtain PD-Rank, achieving a better Kendall tau than comparing methods, even for 10\% of wrong comparisons in simulated data matching our data model and without the assumption of strong connectivity. In real data, PD-Rank requires less computational time to achieve the same Kendall tau than active learning methods.
\end{abstract}

\section{Introduction}
If early applications of the ranking problem were focused on voting scenarios \cite{article:Borda_1781,article:Condorcet_1785} and later on ranking of a small set of items \cite{article:PCM_EM}, modern needs call for the rank of a large number of items from noisy labels and missing information. We can find a wide number of applications ranging from ranking players in different kinds of competitions \cite{article:PCM_incomplete_tennis}, document retrieval \cite{article:Doc_retrieval}, recommender systems \cite{article:recommender_sys}, selection of patients \cite{article:kidney_transplant} or mapping urban safety perception \cite{article:Mapping_Urban_Perception}. 
\begin{figure}[tb!]
	\centering
	\includegraphics[width=0.9\columnwidth]{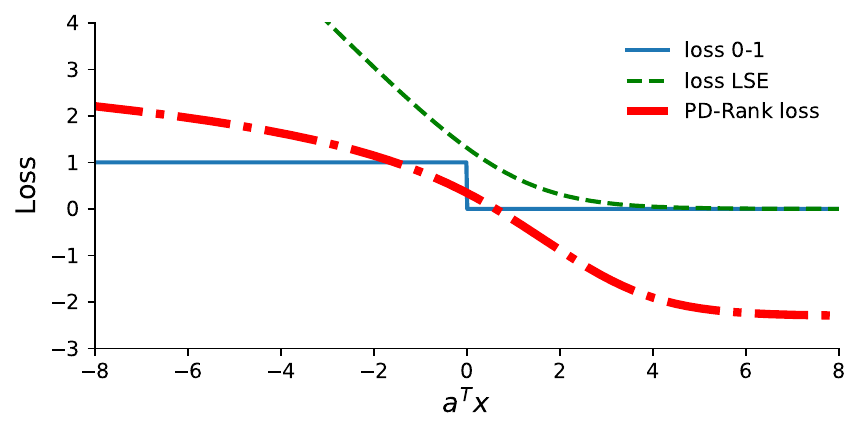}
	\caption{Comparison between the 0-1 loss, LSE and our PD-Rank loss (the logarithm of LSE). On the horizontal axis is the score difference between two items, according to the observed label, $a$. If the value $a^Tx$ is positive, the current ranking score respects the label, otherwise, the order is incorrect. While the 0-1 loss just distinguishes between correct and incorrect ordering, both approximations penalize according to the ``degree of correctness''. However, for values closer to and less than 0, the proposed PD-Rank loss is a tighter approximation to the 0-1 loss, being more robust to incorrect outliers.}\label{fig:loss}
\end{figure}

We consider the problem of ranking a complete set of items given partial rankings of subsets of these items (denoted as labels). Those labels are incomplete, possibly noisy orderings between groups of elements (listwise) or just pairs (pairwise). While pairwise orderings contain less information and require more observations, they are appealing due to their ease and speed of acquisition, as well as their tendency to reduce human assessment errors \cite{article:human_pair_class}. Thus, this is an interesting data collection mechanism to obtain labeling from humans, or when the number of items to rank is too large. In this sense, our work distinguishes from others considering access to a ranking of larger subsets of orderings  \cite{article:subset_Topk,article:subset_fotakis2021aggregating}.

Furthermore, the labels can be obtained from one single assessor or from several annotators. The latter is useful as in a lot of cases it is impractical or even impossible to get all comparisons from a single annotator. This motivates tools such as the Amazon Mechanical Turk, where large sets of pairwise comparisons from different annotators can be obtained at a low cost. However, this introduces new challenges, as annotators may assign different labels to the same pair (either because it is a subjective comparison or there are inevitable assessment errors), resulting in noisy labels. This leads to the so-called \textit{rank aggregation} problem, where we are presented with several rankings of the same items and aim at building a unique one.

Solutions to this problem often take a parametric approach, assuming that label noise depends on the scores of the items (if two items are more distant in the rank, then their labels will be subjected to less noise, and vice-versa). A popular model is the Thurstone’s model \cite{article:Thurstone1927} or the Bradley-Terry (BT) model \cite{article:BradleyTerry}, with several pairwise rank aggregation methods built on this assumption \cite{article:Spectral_MLE,article:BT_model,article:subset_Topk,article:asr,article:RankCentrality,article:ILSR,article:HodgeRank,article:rank_SVD_2021}. However, this assumption does not necessarily hold for all applications, and sometimes there is the need for a more general modeling of the noise \cite{article:SVM_RankAggregation_2014, article:NoisySorting_2007, article:EffRank_2013,article:Minimax} --- the approach we take in this work. We consider that noise is the same for any comparison, regardless of the ranking of the items being compared\footnote{This formulation can also be understood as the noisy sorting problem with resampling.}. Nonetheless, a disadvantage of the non-parametric models is that they provide only a ranking (ordered items) rather than a rating (scores for each item). Therefore, it is important to note that our method focuses exclusively on item ranking. 

Other non-parametric approaches to the pairwise rank aggregation problem are either not suitable for a large number of items (or only focus on retrieving the top-K items), are unable to accomodate individual annotator behaviour or assume that labels can be selected in an active manner, which may not always be possible. We expand on this in Section~\ref{sec:RelatedWork}.In this work, we focus on addressing this gap by developing a principled model for the rank aggregation problem from possibly noisy pairwise comparisons, made by multiple annotators (with possibly different behaviour). We propose an approximation that is close to the initial formulation and we propose an optimization method that scales well with the number of items and aims at retrieving the ranking of the full set of items.

\subsection{Contributions}
Our main contributions are:
\begin{itemize}
    \item we propose a \textbf{new formulation for the ranking problem}, based on a well-defined data model, that allows for personalized annotator behavior (Section~\ref{sec:ProbStatement});
    \item we propose a \textbf{sum of quasi-convex approximations} thus avoiding the drawbacks of the original optimization and a \textbf{fast specialized algorithm} that scales well with the number of items (Section~\ref{sec:OptimizationProb}); 
    \item we provide a measure of \textbf{confidence on each pairwise comparison}, obtained from weights learned during loss optimization (Section~\ref{sec:Approximation});
    \item when compared with state-of-the-art methods PD-Rank \textbf{outperforms competitors for independent noise model with limited observations} (up to at least 10 standard trials); is still \textbf{competitive for the BT noise model} for connected data and few observations; \textbf{scales well with the number of items}; compared with active learning methods, \textbf{achieves high Kendall values in less computational time}.
\end{itemize}

\section{Related work}\label{sec:RelatedWork}

\paragraph{Large scale solutions.} In the rank aggregation literature, we can find two main trends to address the large scale setting: either to carefully sample the observations so that a smaller number of them provides more information \cite{article:Crowd_BT_2013,article:Hybrid_MST_2018,article:Hodge-active_2017,article:ActiveRank_Heckel_2016}, or to use suitable optimization solutions, able to handle large data. The latter usually consider as input pairwise comparison matrices containing a ratio between each pair of items \cite{article:PCM_iter_LargeScale_Sparse}. Since individual labels are not directly incorporated, integrating annotator behavior becomes challenging. The former line of work relies on active learning techniques to sequentially pick informative pairs.  Except for the work in \cite{article:ActiveRank_Heckel_2016}, where the number of comparisons won at one step is used to select the next pairs to be compared, the other authors employ parametric assumptions. Both Crowd-BT \cite{article:Crowd_BT_2013} and Hybrid-MST \cite{article:Hybrid_MST_2018} are based on the Bradley-Terry model, with the latter having smaller time complexity and the former including modeling of annotator behavior. Hodge-active \cite{article:Hodge-active_2017} employs the HodgeRank model as well as the Bayesian information maximization to actively select the pair. Apart from the latter \cite{article:Hodge-active_2017}, which presents an unsupervised approach, the previous methods all assume that the pairs can be selected in an active manner, which may not always be the case. This is a motivation to look for other approaches, while considering that they are not mutually exclusive. 

\paragraph{Top-K retrieval and learning to rank.} In the context of large scale rank aggregation, there has been a growing interest in methods tailored for Top-K rank \cite{article:subset_Topk,article:rating_topk,article:adversarial_topk}, i.e. the retrieval of highest K items of a set. However, when the goal is to retrieve the full rank, some of the assumptions underlying these methods no longer hold. There is also extensive research on the related problem of learning to rank \cite{article:LargeScale_LearningRank,article:LargeScale_LearningRank_2016, article:ranknet2lambdarank, article:Doc_retrieval}, where the goal is to predict the ranking of unseen items based on a set of features. Instead, we do not account for item features and focus exclusively on ranking observed items.

\paragraph{Non-parametric models.} Looking in more detail at the non-parametric line of work \cite{article:SVM_RankAggregation_2014, article:NoisySorting_2007, article:EffRank_2013,article:Minimax}, where the data model is similar to ours, the approaches diverge in the underlying assumptions. One can consider that all possible comparisons are available (or not) and that we have access to repeated observations (or not). For modern applications, it seems more reasonable to assume that we have access to repeated observations, but do not necessarily contemplate all possibilities (given the large number of items being compared). In \cite{article:NoisySorting_2007} the authors assume all possible pairs of comparisons are known, without resampling, while in \cite{article:EffRank_2013} they do not assume access to all comparisons, but still do not consider resampling. In \cite{article:Minimax}, the data model is the same as ours, as they allow for resampling and consider only partial access to the observations, thus making it appealing for large scale item datasets. However, their formulation with a pairwise matrix aggregates all comparisons for the same pair, while ours allows for a specification of a different level of confidence on each pair. While this contributes to ranking accuracy it makes our approach dependent on the number of available comparisons, while \cite{article:Minimax} only depends on the number of items. Finally, we look into SVM-RankAggregation \cite{article:SVM_RankAggregation_2014}, as the closest approach both in terms of data model and problem formulation. The authors propose an SVM algorithm, which is proved to converge to the optimal ranking under their most general condition on the pairwise comparison matrix, so-called generalized low-noise (GNL). Under this assumption, the induced dataset is linearly separable and a hard-margin SVM is used, otherwise one uses a soft-margin one with a suitable regularization parameter. However, the experiments are performed for a small number of items and comparisons. Besides, we test our method in a more general setting than the GNL. Another non-parametric approach is the Borda Count analyzed in \cite{article:SimpleRobust_Shah_2018}, where items are simply ranked by the sum of their pairwise comparison probabilities. This simple approach leads to a much lower computational load, but struggles to handle noisy and incomplete data.

\section{The PD-Rank Model}
\label{sec:ProbStatement}

Given $M$ elements and $N$ pairwise comparisons between them, the goal is to retrieve their true unknown ranking encoded by $x \in \mathbb{R}^M$. Note that elements are given in an arbitrary order, which must not be confused with their ranking, i.e., $x[i]$ is the score of element $i$, so that the ranking is achieved by sorting $x$. Comparison $n$ between item $i$ and item $j$, for $i>j$ is encoded in a sparse vector $c_n \in \mathbb{R}^M$, with
\begin{equation*}
	c_n^{(i)} = 1, \qquad c_n^{(j)} = -1,
\end{equation*}
and the remaining elements set to zero. Each vector $c_n$ corresponds to one comparison, with a total of $N$ vectors. Multiple observations between the same elements are allowed. The observed label we have access to is only an ordering of the two items, that is,
\begin{equation*}
	y_n^{*} = \mathrm{sign}(c_n^Tx).
\end{equation*}
where $\mathrm{sign}(x)$ is the sign function and $y_n^*$ is the true label.

Meaning that, if item $i$ has larger ranking than $j$, then $y_n=1$; if there is a tie $y_n = 0$; and if $j$ ranks lower, then $y_n=-1$. However, we do not have access to this label but rather to noisy ones, according to the assumptions expressed in the previous section. Therefore, we model the noise $z_n \in \{-1,1\}$ as an independently and identically distributed Bernoulli random variable, where the probability of a toggle error is $\delta_n$, such that
\begin{equation}
	\mathbb{P}(z_n = -1) = \delta_n, \qquad \mathbb{P}(z_n=1) = 1 - \delta_n.
	\label{eq:toggle_error}\end{equation}
Therefore, the noisy observed labels are given as  $\label{eq:noise_model}
y_n = \mathrm{sign}(c_n^Tx)z_n$.
\subsection{Maximum likelihood estimation}
Given the previous model, we formulate  the problem as finding the maximum likelihood estimator for the ranking $x$, given a set of $N$ observations $y_n$. Under the assumption that the comparisons are i.i.d.\ and stacking them in a vector $y = (y_1, \cdots, y_N)$, the likelihood is given as
\begin{equation}
	p(y|x) = \prod_{n=1}^N p(y_n|x),\label{eq:ML}
\end{equation}
where the likelihood for each observation is
\begin{equation*}
	\begin{split}
		p(y_n|x) =& p(y_n |x,z=1)\mathbb{P}(z_n = 1) +\\  &p(y_n |x,z=-1)\mathbb{P}(z_n = -1).
	\end{split}
\end{equation*}
The conditional probabilities translate to the probability of observing a given label if an error occurred or not, given the ranking knowledge. This corresponds to either $0$ or $1$, depending on whether $y_n$ corresponds to the true comparison $c_n^Tx$ or not. We express this with an indicator function 
\begin{equation*}
	\mathbb{I}_S(u) =
	\begin{cases}
		1 & \mathrm{ if }\quad  u \in S\\
		0 &  \quad \mathrm{otherwise}.
	\end{cases}
\end{equation*}Therefore, and considering the error probabilities as given in \eqref{eq:toggle_error}, we write the likelihood as
\begin{equation}
	\begin{split}
		p(y_n|x) = & \mathbb{I}_{\mathrm{sign} (c_n^Tx)}(y_n)(1-\delta_n) + \\ &\mathbb{I}_{-\mathrm{sign} (c_n^Tx)}(y_n)\delta_n
	\end{split}
	\label{eq:likelihood_obs}\end{equation}
and the maximum likelihood estimator of \eqref{eq:ML} is given as 
\begin{equation}
	\hat{x}_{ML}\in \mathrm{argmax}_x  \sum_{n=1}^N
	\log  \left   \{ p(y_n|x)  \right \}. \label{eq:Prob_MLE}
\end{equation} Our goal is to find an approximate solution for this optimization problem. 

\section{Solving the PD-Rank Problem}
\label{sec:OptimizationProb}

In this section we describe our proposed algorithm: PD-Rank. We first note that our problem can be reformulated into the well-known 0-1 loss. Problem \eqref{eq:Prob_MLE} can be reformulated as
\begin{equation*}
\minimize_x \sum_{n=1}^N\log \Big( \frac{1-\delta_n}{\delta_n}  \Big) [1-\mathbb{I}_{\geq 0}(y_nc_n^Tx)].\end{equation*}
We will introduce a variable $w_n = \log \Big( \frac{1-\delta_n}{\delta_n}  \Big)$,  now representing the confidence on each observation, instead of the noise associated to it. We will also introduce $a_n = y_nc_n$, corresponding to the noisy data received from annotators. Therefore, we will obtain 
\begin{equation}
\minimize_x \sum_{n=1}^Nw_n [1-\mathbb{I}_{\geq 0}(a_n^Tx)],
\label{eq:likelihood_01loss}\end{equation}
which is the formulation of the well-known 0-1 loss. 

\subsection{Approximation}\label{sec:Approximation}

Problem~\eqref{eq:likelihood_01loss} is nonconvex and discontinuous, and thus difficult to optimize. One way to approach this problem is to use an easy-to-optimize approximation, desirably as tight to the original~\eqref{eq:likelihood_01loss} as possible. A common convex surrogate of the 0-1 loss in the Hinge loss \cite{article:hingeloss}, given as 
\begin{equation*}
\minimize_x \sum_{n=1}^Nw_n \max\{0,1-a_n^Tx\}.
\end{equation*} However, given the nature of our problem, where $a_n^Tx$ is always the difference between two terms of $x$, this cost will favor smaller values of $x_n$ instead of the correct rank score. That is, in the presence of opposite labels for the same pair, one of them will necessarily lead to $a_n^Tx<0$, regardless of the values in $x$. On the other hand, the minimum of each term is $0$, so we strongly penalize large differences in wrong ranks but equally benefit from any difference in correct rank. 

Given that this surrogate is not appropriate for our setting, we propose a different approximation using the logarithm of the Log-Sum-Exp (LSE) function ($LSE(t) = \log (1+e^t)$), moving away from convexity but attaining continuity and differentiability. We notice that using the LSE would lead to a similar model as the one found in BT, but in order to stay closer to our noise assumptions we take the logarithm of LSE, which is a tighter approximation to the 0-1 loss (we refer the reader to Figure~\ref{fig:loss} for a more detailed comparison between the original cost and the two approximations). Consequently, we can overcome the problems with the hinge loss in a related manner to the BT while staying closer to our noise model (and as we will show in Section~\ref{sec:Experiments} this will have a positive effect on the ranking accuracy). 

At this point, we leave $w_n$ temporarily aside (i.e. we consider constant and equal $w_n$ for all terms) and propose the following approximation for the 
data fidelity term $(1-\mathbb{I}_{\geq 0}(a_n^Tx))$ in 
\eqref{eq:likelihood_01loss} 
\begin{equation}
\begin{split}
	& \minimize_x \sum_{n=1}^N \log \Big[ \log\big(1+e^{1-a_n^Tx}\big) + \epsilon\Big] +  \gamma\|x\|^2_2\\
	&\subjto  \quad \textbf{1}^Tx = 0,\label{eq:NewCost}
\end{split}\end{equation}
where $\epsilon>0$ is small and the constraint is added to anchor the solution (without it, adding any constant to $x$ still leads to a solution of the problem). The regularization term $\gamma\|x\|^2_2$ is also added to prevent the unbounded increase of distance between the elements in $x$, with the regularization weight $\gamma$ set to a small value.

To solve the approximation we will use an iterative re-weighted approach \cite{article:IterReweighted,article:Iteratively_1,article:Iteratively_2,article:Iteratively_3,article:Iteratively_4}, where the weights are taken as a measure of confidence in our observation. We take this approach due to the quasi convexity of the $\log \log $ terms in \eqref{eq:NewCost}, that will be linearly approximated in the sequence, similarly to the approach in the sparse reconstruction with $L_p$ ($0<p<1$) quasi norms literature. We linearize the $\log$ term in \eqref{eq:NewCost}, while leaving the regularization term untouched, as it does not present the same complexity. So, taking the Taylor expansion of the $n$-th $\log \log $ cost term we get
\begin{equation*}
\log \Big[ LSE(1-a_n^Tx) + \epsilon\Big] \approx   \frac{ \log\big(1+e^{1-a_n^Tx}\big)}{ \log\big(1+e^{1-a_n^Tx^k}\big)+\epsilon}.
\label{eq:TaylorApp}\end{equation*}
Therefore, according to \cite{article:IterReweighted}, at each update $k$, $x^k$ belongs to 
\begin{equation}
\begin{split}
	\argmin_{x \in \reals^m} & \sum\limits_{n=1}^N \omega_n^{k} \log\left (1+e^{1-a_n^Tx} \right) + \gamma \|x\|^2_2 \\ 
	\text{subject to } &\ones^Tx = 0,\label{eq:cvxSubproblem}
\end{split}\end{equation} where 
\begin{equation}
\label{eq:omega}
\omega_n^{k} = \frac{1}{LSE(1-a_n^Tx^{k-1}) + \epsilon}.
\end{equation}
As expressed in Proposition~\ref{prop:min}, Problem~\eqref{eq:cvxSubproblem} has always a unique minimizer. This entails that whenever the algorithm is run with a given dataset, the ranking variable $x$ returned will always be the same for the same data.

\begin{proposition} \label{prop:uniqueness}
Let problem~\eqref{eq:cvxSubproblem}, be one instance for a generic $k > 0$, and define $\omega^0 = (\omega_1^0, \cdots, \omega_N^0)$ as the vector collecting all the initialization weights for the iterative reweighting optimization scheme such that $\omega^0 = \ones$. Define the comparison noisy data $a_n = z_n c_n$, a small regularization constant $\gamma > 0$, and $\omega_n^k$ defined as in \eqref{eq:omega}.
	Then, the $\argmin$ set is a nonempty singleton set, i.e., problem~\eqref{eq:cvxSubproblem} always has a solution and the solution is unique.\label{prop:min}
\end{proposition}
\begin{proof}
	The result follows easily from convexity of the summation terms, where 
	\begin{equation*}\omega_n^k \log\left (1+e^{1-a_n^Tx} \right)
	\end{equation*} is the product of a positive number and the convex log-sum-exp,
	 and the $\gamma$-strong convexity of the second term $\gamma \|x\|^2_2$ \cite{boyd2004convex}.
\end{proof}

We now note that $\omega_n$ and $w_n$ evolve in the same way, taking into account the uncertainty associated with observation $n$. That is, for small $\delta_n$, $w_n =  \log \Big( \frac{1-\delta_n}{\delta_n}  \Big)$ will take higher values, growing to $+\infty$, while for larger values (up to $1/2$) it will decrease to zero. So, $w_n \in \big[0,+\infty\big]$ giving a lower weight to terms with high uncertainty and vice-versa. In the same way, when $a_n^Tx^k$ takes increasingly negative values (corresponding to increased uncertainty with the majority of the data), $\omega_n^k $ will tend to $0$, while in the opposite case it will go to $+\infty$. Therefore, it is reasonable to establish a connection between both these terms, and we will interpret $w_n$ according to $\omega_n$, the final iterate of $\omega_n^k$ (visual support of this explanation can be found in \cite{preprint:rank}). Consequently, we note that the method does require a known knowledge of the noise levels $\delta_n$.

Finally, note that all coherent observations of the same pair (i.e. attribute the same ordering of those two items) will have the same weight. Therefore, they can be replaced by a single term with an additional weight factor corresponding to the number of times it was observed. Essentially this means that repeated observations will not contribute to added computational time. So, the algorithm has a general formulation easily suited for both the setting of annotator behavior modeling and the one without it, allowing for less computational complexity in the latter. The next section will detail the algorithm used to solve each sub-problem~\eqref{eq:cvxSubproblem}.

\subsection{PD-Rank for large-scale data}

\subsubsection{Background: Primal-Dual Hybrid Gradient.}

A useful tool to solve large-scale convex problems is the Primal-Dual Hybrid Gradient (PDHG) algorithm \cite{zhu2008efficient,pock2009algorithm,article:PDHG,chambolle2011first}. It tackles problems of the form 
\begin{equation}
\minimize_{x\in \mathbb{R}^m}  g(Ax) + f(x),
\label{eq:Form_PDHG}
\end{equation}
where $g$ and $f$ are closed proper convex functions and $A\in \mathbb{R}^{n\times m}$. The steps to solve the previous problem are given as 
\begin{equation}
\begin{split}\label{eq:PDHG_algo}
	&p_n = \textrm{prox}_{\tau f}(x_n-\tau (A^T v_n))\\
	&q_n = \textrm{prox}_{\sigma g^*}(v_n+\sigma A(2p_n-x_n))\\
	&(x_{n+1},v_{n+1}) = (x_n,v_n)+\lambda_n ((p_n,q_n)-(x_n,v_n)).
\end{split}
\end{equation}
Convergence guarantees are given in \cite{article:PDHG} for step sizes $\tau\sigma  \leq \frac{1}{\|A\|^2_S}$ and $(\lambda_n)_{n \in \mathcal{N}} \in (0,2)$, where $\|A\|_S$ is the spectral norm of $A$.

\subsubsection{Reformulation of our problem}
Problem \eqref{eq:cvxSubproblem} can be expressed as
\begin{equation}\label{eq:myg}
\begin{split}
	\minimize_{x} & \quad g(Ax)+\gamma\|x\|^2_2\\
	\subjto & \quad \textbf{1}^Tx = 0.
\end{split}
\end{equation}
where $	g(y) =  \sum_{n=1}^N \omega_n \log(1+e^{1-e_n^Ty})$
and $A$ is the matrix with row $n$ corresponding to $a_n$ and $e_n \in \mathbb{R}^M$ is a vector of zeros with a $1$ the $n$-th position.
Finally, we can have the equivalent problem 
\begin{equation}
\minimize_{x}  \quad g(Ax) + f(x),
\label{eq:Prob_form_PDHG}
\end{equation}
where $f(x) = i_{\{u:\textbf{1}^Tu=0\}}(x) + \gamma\|x\|^2_2$, with
\begin{equation}
i_s(u) = \begin{cases}
	0 & \mathrm{ if }\quad  u \in S\\
	+\infty  & \quad \mathrm{otherwise}.
\end{cases}\label{eq:myf}
\end{equation}

\subsubsection{PD-Rank algorithm}
Problem \eqref{eq:Prob_form_PDHG} is in the form of \eqref{eq:Form_PDHG}, so we can directly apply the steps in \eqref{eq:PDHG_algo}. The proximal operators are given in Proposition~\ref{prop:Prox} (additional details may be found in Appendix~\ref{sec:Appendix_Proof}).
\begin{proposition}
The proximal operators $\textrm{prox}_{\tau f}(x)$ and $\textrm{prox}_{\sigma g^*}(x)$ of functions $f$ as defined in \eqref{eq:myf} and  $g$ as defined in \eqref{eq:myg} are respectively given as  
\begin{equation*}
	\begin{split}
		\textrm{prox}_{\tau f}(x) &= x' - \bar{\textbf{x}'}\\
		\textrm{prox}_{\sigma g^*}(x) &= x -  \tau   ( \textrm{prox}_{\tilde{w}_n \tilde{g}_n}(\tilde{x}_n) )_{n=1}^N,
	\end{split}
\end{equation*}where $x' = \frac{x}{1+2\gamma}$, $\bar{\textbf{x}}$ is the mean value of $x$, $\tilde{x}_n = x_n/\sigma$, $\tilde{w}_n = w_n/\sigma$ and  $( f_n(x_n) )_{n=1}^N$ is defined as the concatenation of function $f_n$ evaluated at component $n$ of $x$, for its N elements. Furthermore, $\textrm{prox}_{\tilde{w}_n \tilde{g}_n}(\tilde{x}_n)$ is the solution of the following equation
\begin{equation}
	\frac{-e^{1-u}}{1+e^{1-u}}+\frac{u-\tilde{x}_n}{\tilde{w}_n}=0.\label{eq:prox}
\end{equation}\label{prop:Prox}\end{proposition}

\section{Experiments}
\label{sec:Experiments}

PD-Rank\footnote{Code available at \url{https://github.com/FilVa/PD-Rank}} is targeted for large-scale scenarios where the number of observations available is relatively small and the number of items $m$ is large, under a generic non-parametric noise model. It is common to express the number of observations $n$ in terms of \textbf{standard trials}, where one standard trial corresponds to all possible pairings of the $m$ items, given as $m(m-1)/2$  (note that due to repetitions 1 standard trial does not necessarily mean that we observe all possible pairs). In a large data setting we can expect to have access to $n \ll 1$ standard trial.

Since we do not aim at retrieving the correct rating score, but only the ranking, we choose Kendall's tau coefficient as a metric. Given the ground truth scores ($x_{\textrm{GT}}$) and the resulting order given by each method ($x_{\textrm{pred}}$), Kendall's tau is given as
\begin{equation*}
	\tau(x_{\textrm{GT}},x_{\textrm{pred}}) = \frac{P-Q}{P+Q}
\end{equation*}
where $P$ is the number of concordant pairs and $Q$ the number of discordant pairs. A value closer to $1$ implies a better predicted ranking.

We consider simulated data under our noise model and the BT, for different values of $n$ and $\delta$, to test the range of settings where our method is better or equivalent to state-of-the-art approaches. We then increase the number of items to test how well PD-Rank escalates in terms of computation time and accuracy. Finally, we test with real-data, comparing against active learning methods. Additional experiments and metrics may be found in Appendix.

\subsection{Simulated data}
\label{subsec:EXP_Algo}

\paragraph*{Setting.}
The standard setting considers $m=30$ items, from which $n$ pairwise comparisons are randomly picked with uniform probability, with a toggle noise of $\delta=0.1$. We compare PD-Rank against the most relevant methods already pinpointed in the literature review: RC \cite{article:RankCentrality}, I-LSR \cite{article:ILSR}, ASR \cite{article:asr}, HodgeRank \cite{article:HodgeRank}, Borda \cite{article:SimpleRobust_Shah_2018}, SVM-RankAggregation \cite{article:SVM_RankAggregation_2014} and MS \cite{article:Minimax}. In general, we take the largest connected component of the underlying graph, but given that some methods (I-LSR and RC) require strong connectivity, we also conduct simulations under this scenario.

\paragraph*{Weights as a measure of confidence.}
An important feature of PD-Rank is the confidence provided through the weights. To attest this, we look at the evolution of weights over the outer iterations (Figure~\ref{fig:weights}). When there are enough comparisons to produce a completely accurate ranking, all incorrect labels have a weight lower than one and vice-versa, thus perfectly indicating how accurate such measurements are. In the presence of fewer observations, some of the pairs have incorrect weights. Nonetheless, we see that their value is in general closer to 1 than the extremities, thus still providing a measure of uncertainty. 

\begin{figure}[t]
	\begin{subfigure}{0.45\columnwidth}
		\includegraphics[width=\columnwidth]{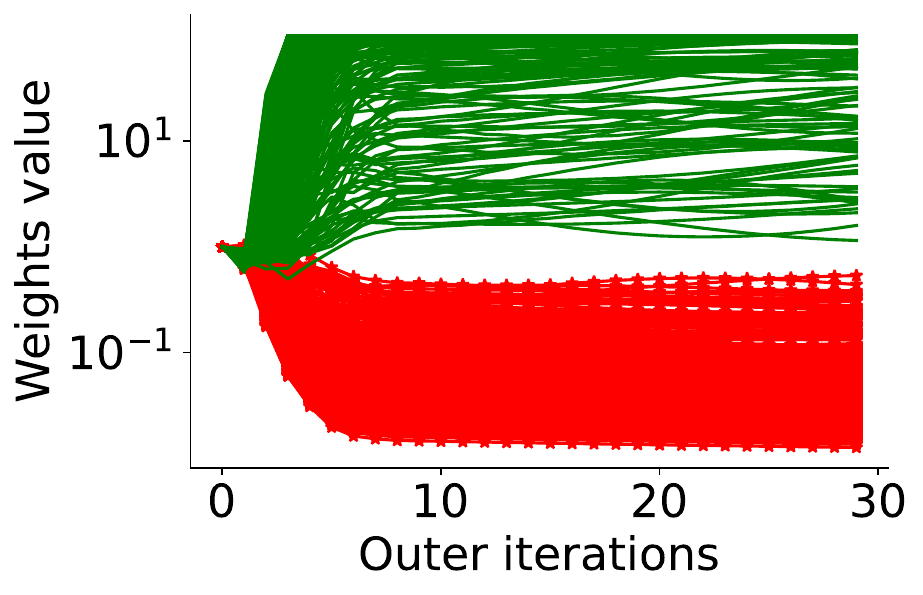}
		\caption{30 standard trials}		
	\end{subfigure}
	\begin{subfigure}{0.45\columnwidth}
		\includegraphics[width=\columnwidth]{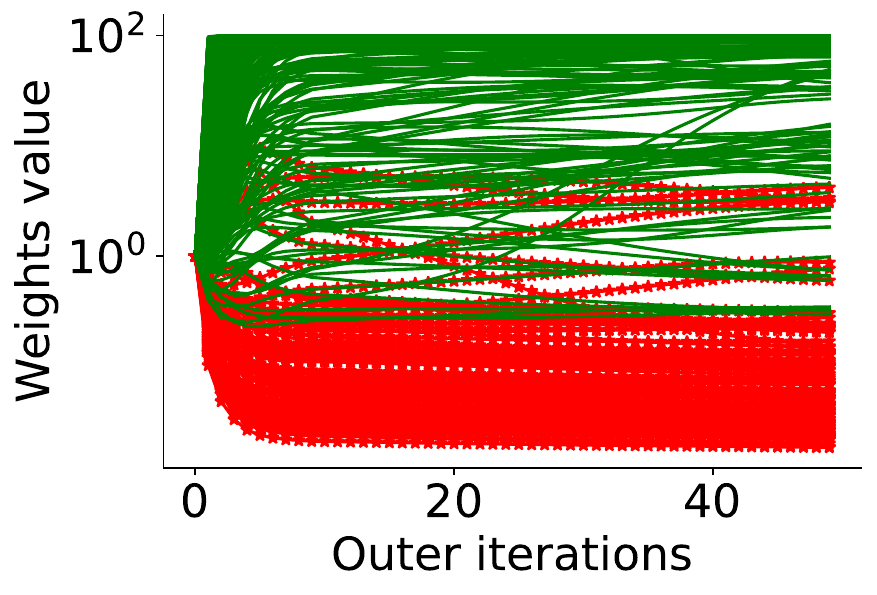}
		\caption{3 standard trials}		
	\end{subfigure}
	\caption{For each panel, we depict the evolution of weights for each unique comparison, over the outer iterations. In green the correct labels and in red the noisy ones. It is noticeable that incorrect observations attain values lower than one, while correct ones tend to $1/\epsilon$. As expected, when more observations are available the weight value is more accurate.}\label{fig:weights}
\end{figure}

\paragraph*{Outperforms competitors for independent noise model with limited observations.}
Under our noise model PD-Rank should outperform the competitors, especially for lower values of $n$ as this is the expected scenario of large data. We test this by progressively increasing $n$ from $0.3$ to $15$ standard trials, with the remaining parameters constant. Indeed, PD-Rank stands out in small to moderate comparison cardinality $n$, up to $10$ standard trials, while for higher values most models show similar performance except for ASR, RC and MS (Figure~\ref{fig:ourmodel_n}). The only method with similar performance is SVM-Rank Aggregation for larger values of $n$. It is also relevant to test the range of noise where this behaviour holds, so we increase $\delta$ from $0.1$ to $0.3$, taking $n$ as $0.8$ of a standard trial. PD-Rank presents a gap of almost $0.1$ Kendall tau up to $\delta=0.3$, when ASR and HodgeRank start to approach its performance, but still remain inferior (Figure~\ref{fig:ourmodel_noise}). Therefore, PD-Rank is the best choice for scenarios with few observations available, even under high values of noise, if the data follows the considered noise model.

\begin{figure}[t]
	\begin{subfigure}{\columnwidth}
		\includegraphics[width=\columnwidth]{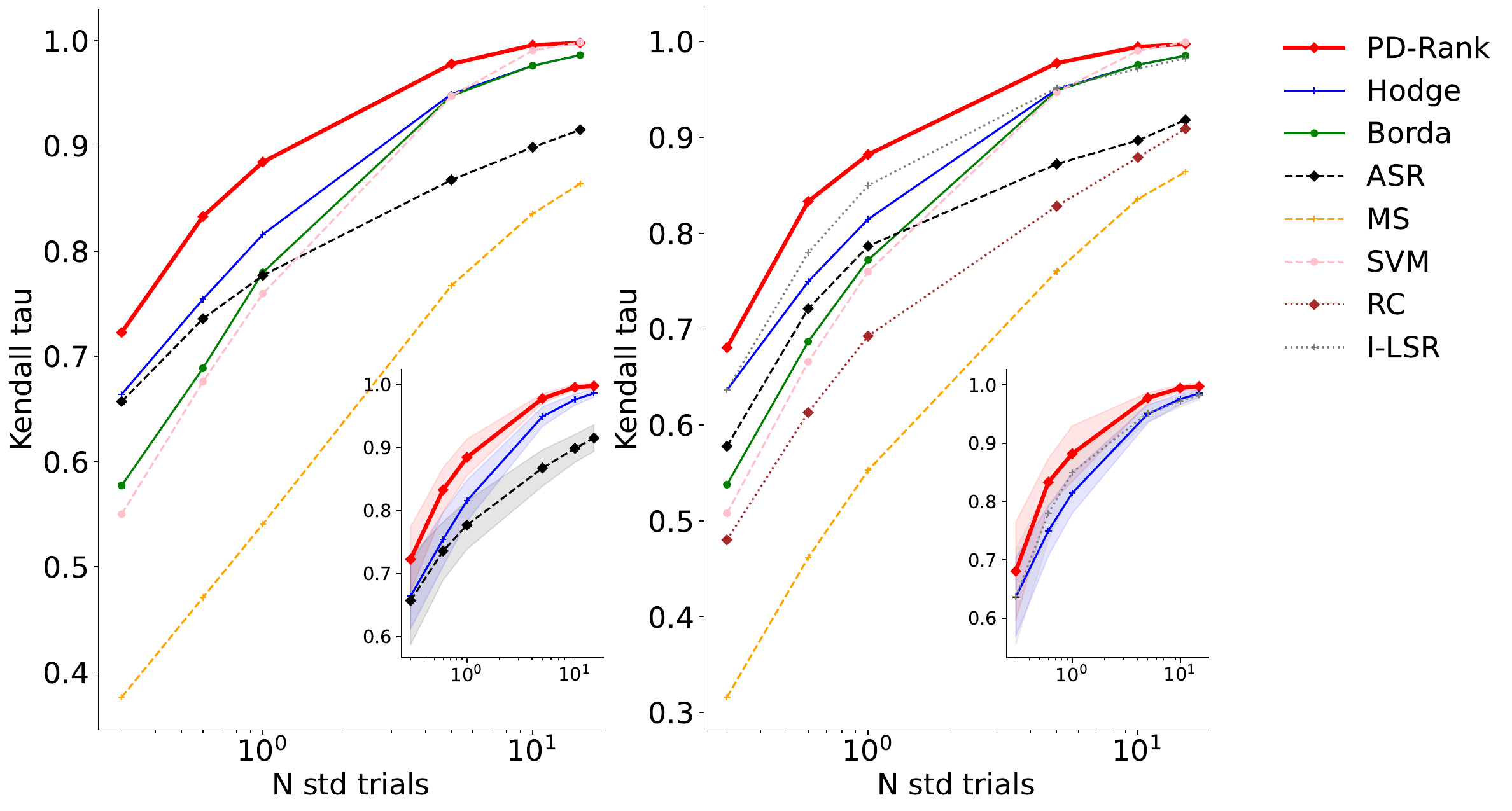}
		\caption{Variation with number of comparisons for the largest connected (left) and strongly connected graphs (right)}\label{fig:ourmodel_n}
	\end{subfigure}
	\begin{subfigure}{\columnwidth}
		\includegraphics[width=\columnwidth]{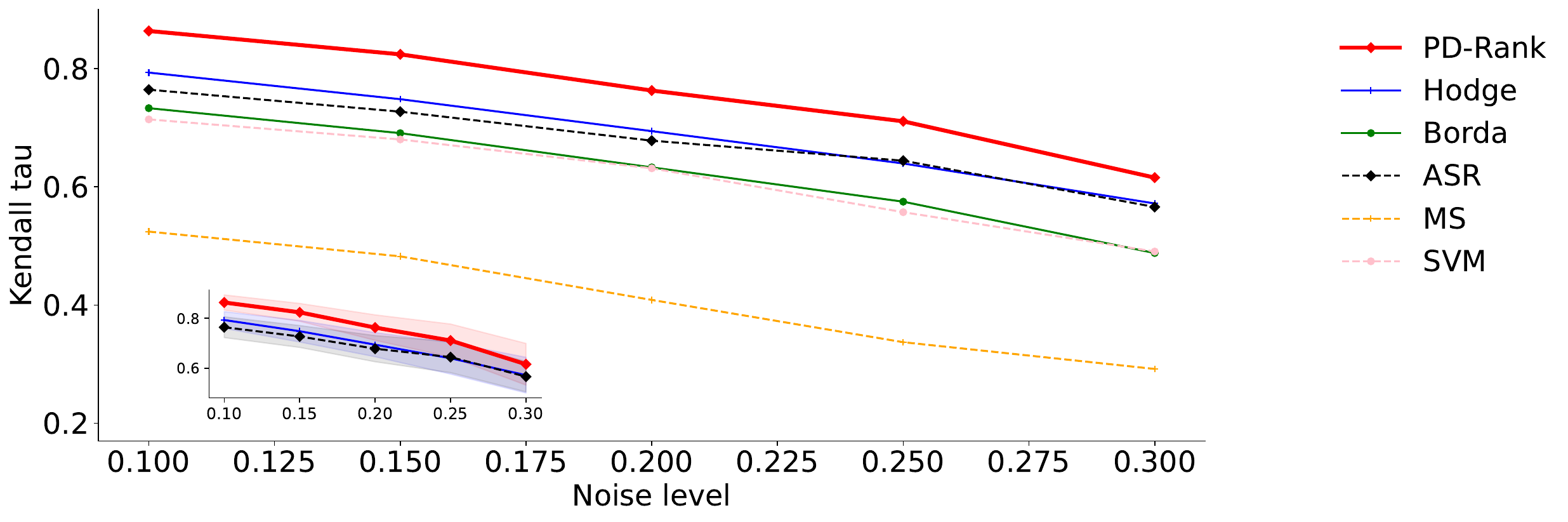}
		\caption{Variation with noise level}\label{fig:ourmodel_noise}
	\end{subfigure}	
	\caption{Kendall tau for simulated data with noise independent of score value. All the methods are depicted in the main plot, while the smaller subplots depict the same values with the 1 standard deviation in the shaded area, for the most competitive benchmarks.}
	\label{fig:ourModel}
\end{figure}
\paragraph*{Competitive for a different noise model.} 
Given that the data in the latter experiment was produced according to the assumptions in our model, a better performance from PD-Rank was expected, so we conduct further simulations under the BT noise model, often found in the ranking literature. The remaining setting is the same and we consider two different ranges of scores for $x$, which roughly translates to different levels of noise in our model (Figure~\ref{fig:BTModel}). For a larger $n$ most parametric methods  naturally outperform the non-parametric ones, as the data follows their assumptions. If the underlying graph is strongly connected, then I-LSR is always the better option for this noise model, but the method is not available for connected graphs (taking the largest strongly connected component is an option but it will likely lead to loss of information). Here we present only the results for strongly connected graphs, but the results for connected ones are similar and may be found in Appendix~\ref{app:BT_connected}. Interestingly, for lower values of $n$ PD-Rank is competitive even with the parametric methods, especially for larger scores, i.e. for lower noise. Since large data scenarios will necessarily contemplate $n \ll 1 \textrm{ standard trial}$ and will unlikely present strong connectivity over a large set of elements, PD-Rank is still a competitive option, even under a different noise model.

\begin{figure}[t]
	\begin{subfigure}{\columnwidth}
		\includegraphics[width=\columnwidth]{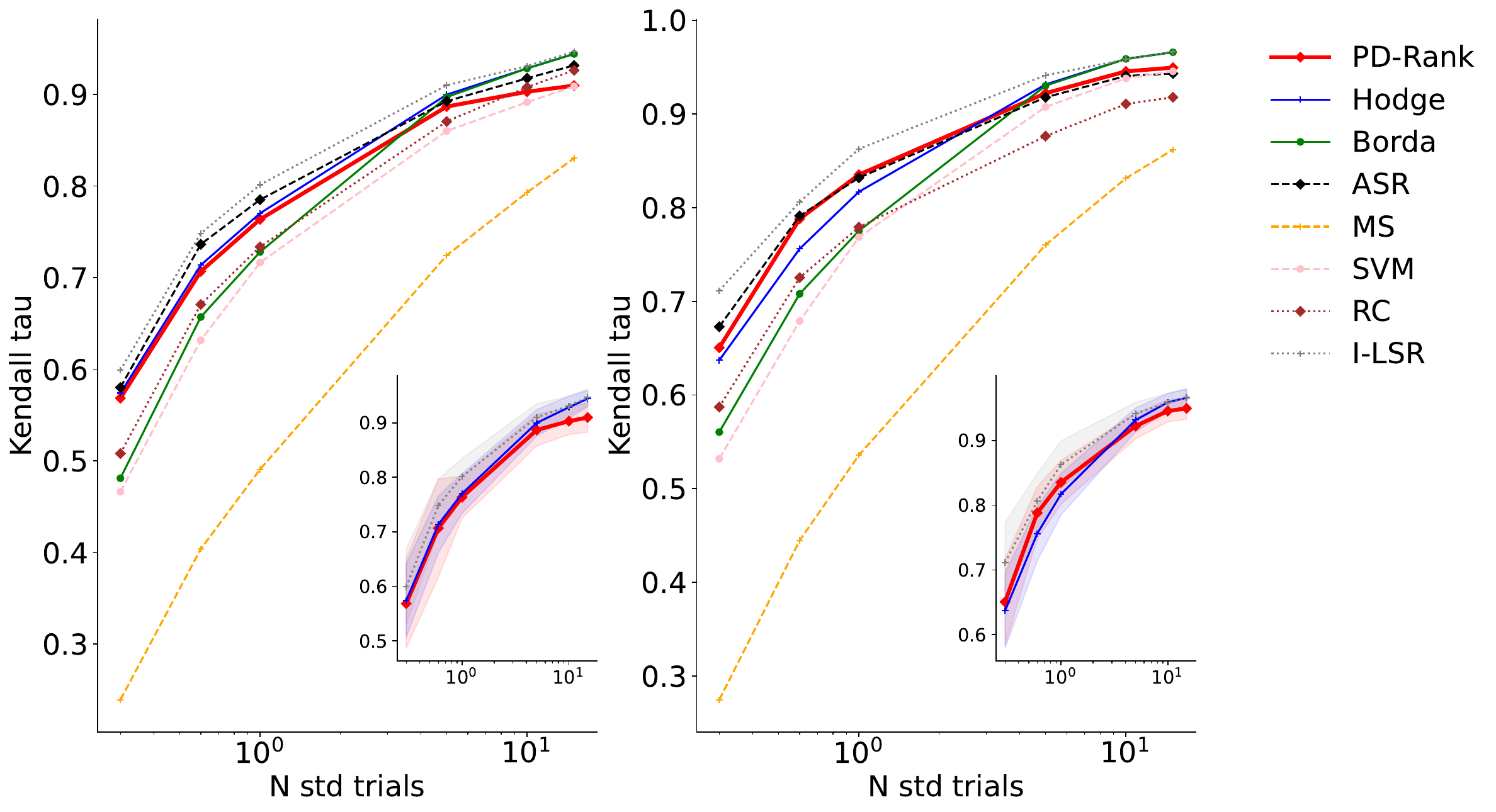}
		\caption{Kendall tau}
	\end{subfigure}	
	\begin{subfigure}{\columnwidth}
		\includegraphics[width=\columnwidth]{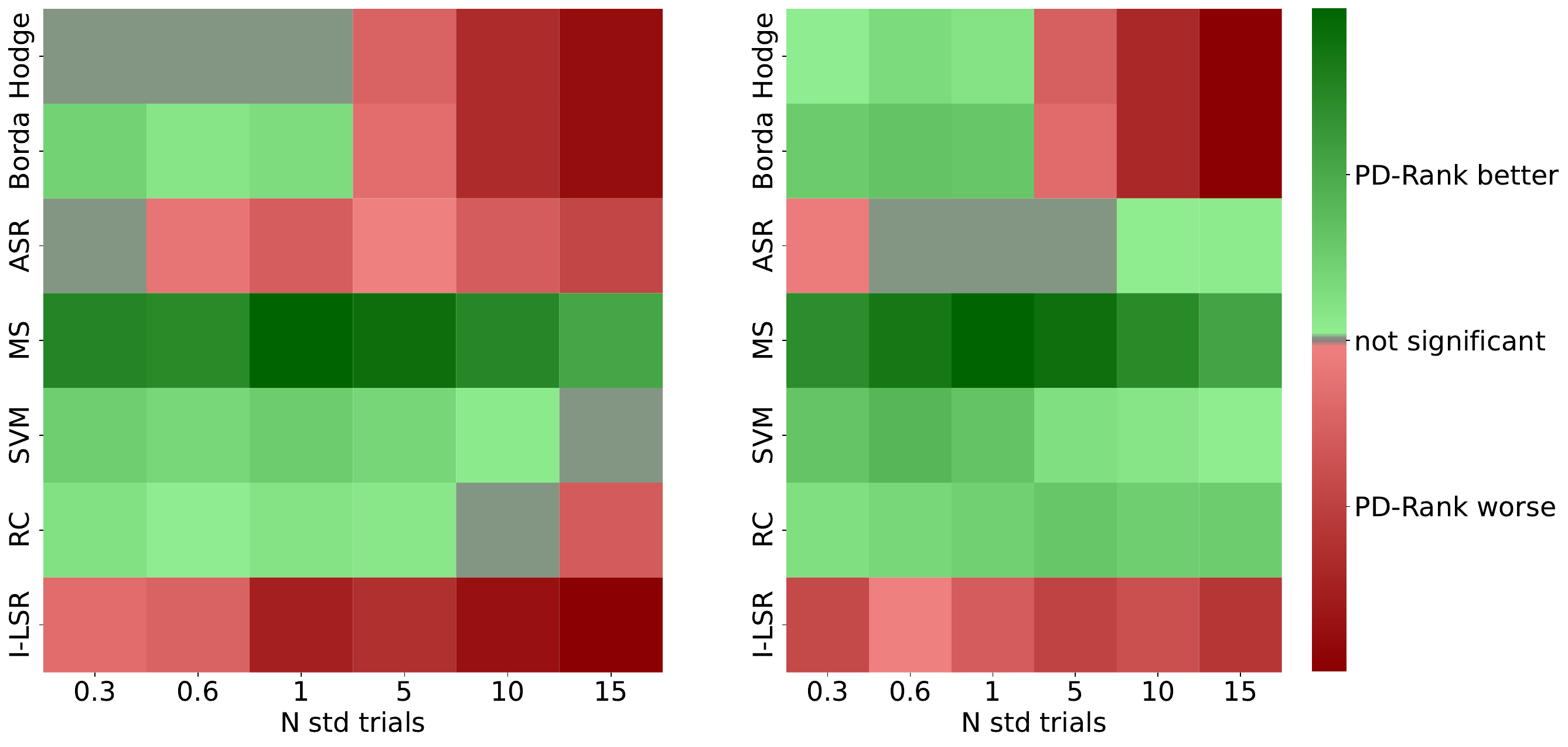}
		\caption{Results of paired t-test between Kendall scores of PD-Rank and other methods}
	\end{subfigure}	
	\caption{Simulated data according to the BT model for strongly connected graphs, on the left scores in range $[0,5]$ and on the right $[0,10]$. For the heatmap, the color is proportional to the statistic value, being green if PD-Rank performs better and red, otherwise; if the p-value is smaller than $0.05$ the color is set to grey. }
	\label{fig:BTModel}
\end{figure}
\paragraph*{Scales with the number of items.}
For the large data scenario, we want to see how PD-Rank handles an increase in number of items $m$ for low values of comparisons $n$. We consider connected graphs and $5$ comparisons per item, taking $m$ between $50$ and $1000$. PD-Rank always stands above the competitors, with a more clear gap as $m$ increases (Figure~\ref{fig:largedata}). Although PD-Rank has the highest computational time it scales well with the number of items, which is the fundamental point. PD-Rank has a non-optimized implementation and the absolute values of computation time have a large potential for improvement, especially when it comes to the solution of equation \eqref{eq:prox}.

\begin{figure}[t]
\centering
	\includegraphics[width=	\columnwidth]{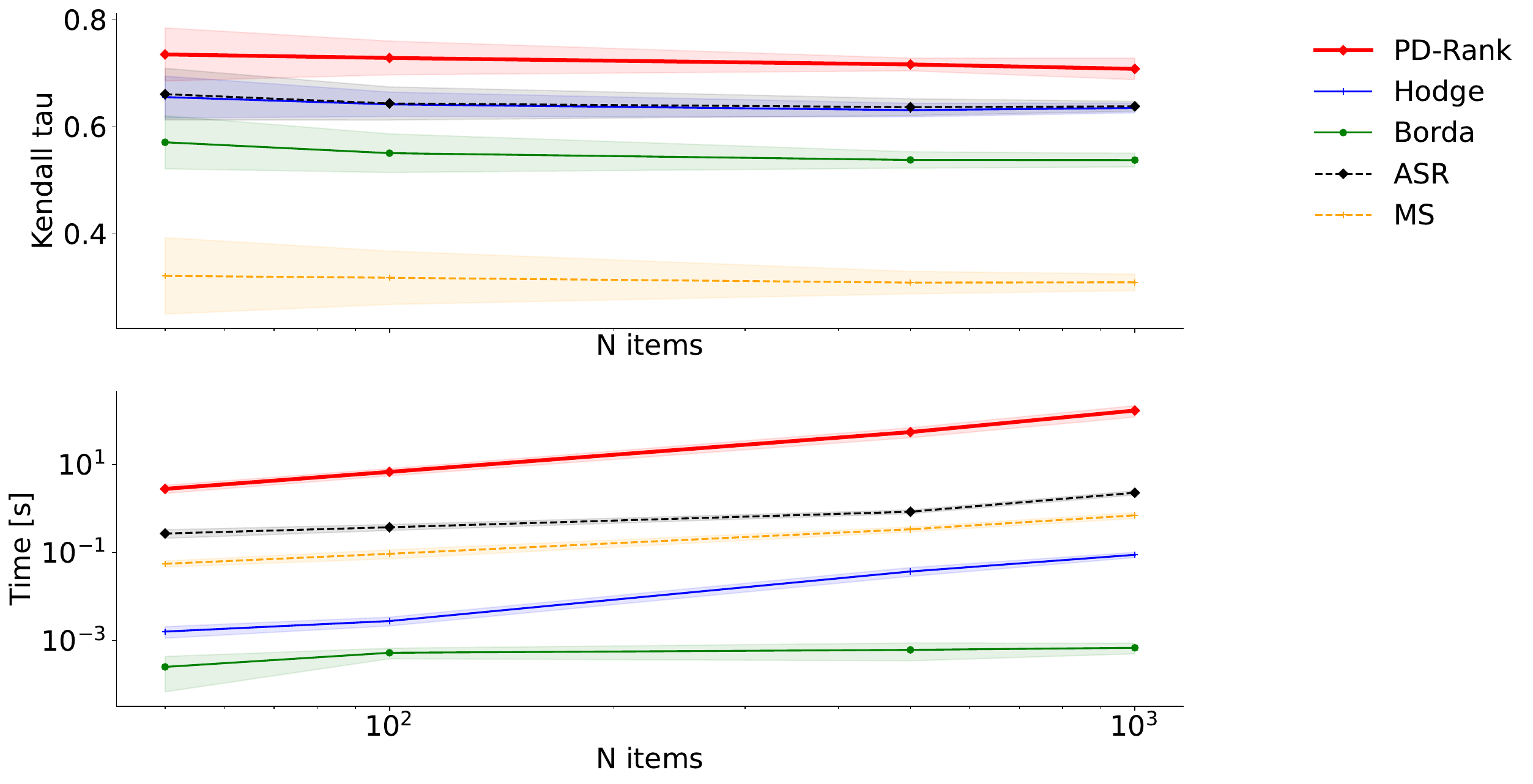}
	\caption{Large data setting, with variation of number of items $m$, with $5$ comparisons per item. On top we see the performance in terms of the Kendall coefficient and on the bottom the respective computation time. }
	\label{fig:largedata}
\end{figure}

\subsection{Real-world datasets}
\label{subsec:EXP_RealWorld}
\paragraph*{Dataset and setting.}
We use two real-world common examples taken for the rank aggregation problem, for the pairwise comparison of distorted videos and images. The \textbf{Image Quality Assessment (IQA) dataset} is taken from the LIVE database \cite{database:LIVE_2008} and \cite{database:IVC_2005} by the authors in \cite{dataset:IQA} and the \textbf{Video Quality Assessment (VQA) dataset} from~\cite{database:VQA} by the authors in~\cite{dataset:VQA}. Both cases have original reference images or videos, respectively, with additional distorted versions, from which the authors produce a dataset of pairwise comparisons. Each reference includes 16 items to be ordered.

We compare PD-Rank with Borda Count and state-of-the-art active learning methods: Hodge-active \cite{article:Hodge-active_2017} and  Hybrid-MST \cite{article:Hybrid_MST_2018}. Since there is no available ground truth (GT) for these datasets it is common practice to take GT as the ranking obtained with the full dataset (e.g. 32 standard trials for one IQA reference set). Note that we will take the BT solution as GT, as it is the golden standard, but this brings some unfairness to Borda Count and PD-Rank, as they are usually not able to reach Kendall of 1.0, since the GT according to all methods is not the same. 

\paragraph*{Compared with active learning methods, achieves high Kendalls in less computational time.} 
Active learning methods take advantage of pair selection to achieve better accuracy with smaller values of $n$, but this necessarily leads to an increased time complexity. We wish to test whether PD-Rank is able to reach similar of higher values of Kendall in less time, even if it needs to resort to a larger number of observations. While Hodge-active and Hybrid-MST always arrive at a higher Kendall with less available comparisons, they must be run sequentially. So, in order to reach Kendall values larger than approximately $0.85$, PD-Rank becomes the best option in computation time (Figure~\ref{fig:VQA-ref1}).  Therefore, if the number of comparisons is limited and the setting allows for selection of pairs, Hodge-active and Hybrid-MST are the best choice; if, for instance, the comparisons were already retrieved and one wishes to obtain a ranking in a shorter time window, then PD-Rank is more suitable.
\begin{figure}[t]\centering
		\includegraphics[width=0.85\columnwidth]{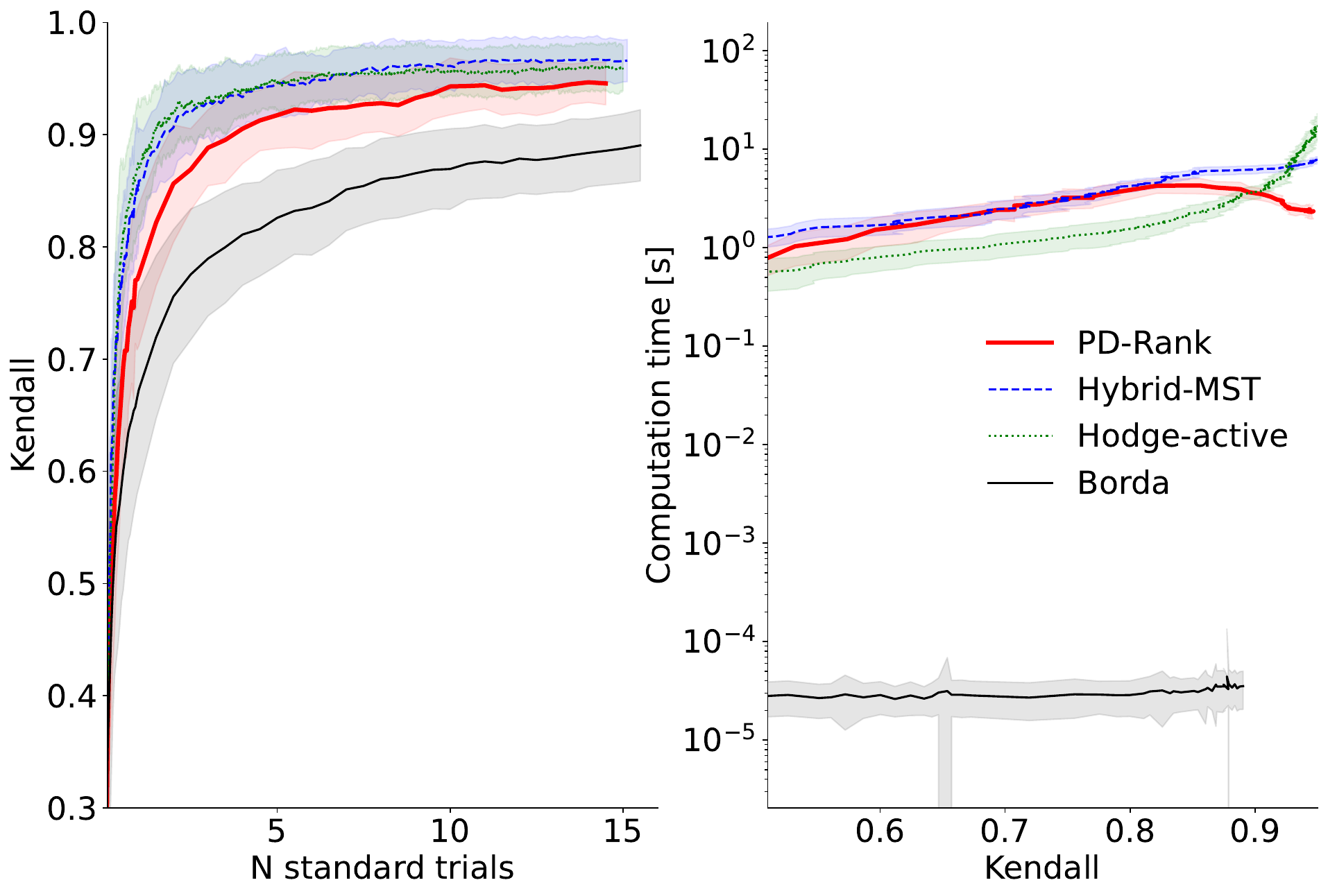}
	\caption{Results for one example of IQA-dataset (other examples can be found in Appendix~\ref{appendix:exp_real}). On the left plot, we can see the evolution of ranking accuracy with the observed samples. On the right we show the evolution of computation time with the achieved Kendall coefficient: lines above PD-Rank require more time for the same accuracy and vice-versa.}
	\label{fig:VQA-ref1}
\end{figure}
\section{Conclusion}
\label{sec:Conclusion}
We defined a new model for the ranking problem, leading to a non-convex and discontinuous optimization problem. We obtained a continuous nonconvex approximation and solved it with PD-Rank, an algorithm based on iteratively re-weighted minimization and the Primal-Dual Hybrid Gradient. The method competes in accuracy with the state-of-the-art under simulated and real-world data experiments, and its scalability to large data scenarios. Nonetheless, a limitation that remains to be addressed in future work is the evaluation of our approach in even larger-scale scenarios ($> 1K$ items) and the inclusion of theoretical guarantees regarding our approximation. Moreover, while our method is well suited for modeling annotator behavior by accounting for varying noise across observations, this was not explored in the present study and requires further research. Future directions also include the integration with active learning, especially taking advantage of the information provided by confidence weights.

\paragraph{Acknowledgments.} We sincerely thank the reviewers for their feedback and suggestions, which have significantly contributed to improve the quality of our work. This work is funded by national funds through FCT – Fundação para a
Ciência e a Tecnologia, I.P., in the scope of project HyCARE (2024.07361.IACDC), UID/04516/NOVA Laboratory for Computer Science and Informatics (NOVA LINCS), UIDB/00297/2020 \url{https://doi.org/10.54499/UIDB/00297/2020}) and UIDP/00297/2020\\  (\url{https://doi.org/10.54499/UIDP/00297/2020}) (Center for Mathematics and Applications).

\bibliographystyle{siam}

\bibliography{pdrank_bib}

\appendix

\newcommand{\expnumber}[2]{{#1} \times 10^{#2}}

\section{Proof of Proposition \ref{prop:Prox}}\label{sec:Appendix_Proof}
We derive the proximal operators of $f(x)$ and $g^*(x)$ in subsections~\ref{sec:prox_f} and~\ref{sec:prox_g}, respectively. For ease of the reader, in subsection~\ref{sec:prox_theory}, we restate some basic background on proximal operators.
\subsection{Proximal operators: background}
\label{sec:prox_theory}
\begin{Definition}[Proximal operator]
	Given a function $f$, the proximal operator $\prox_f: \mathbb{R}^n\to \mathbb{R}^n $ of $f$ is defined as
	\begin{equation}
		\prox_f(x) = \argmin_u\Big( f(u)+\frac{1}{2}\|u-x\|^2_2\Big)
	\end{equation} and the proximal operator of the scaled function $\lambda f$, where $\lambda>0$ is expressed as
	\begin{equation}
		\prox_{\lambda f}(x) = \argmin_u\Big( f(u)+\frac{1}{2\lambda}\|u-x\|^2_2\Big).
	\end{equation}
\end{Definition}
\begin{proposition}[Indicator function]
Given $f(x) = i_C(u)$, where $ i_C(u)$ is the indicator function defined as
\begin{equation}
i_C(u) =\begin{cases}
1 & \mathrm{ if }\quad  u \in C\\
+\inf  &\quad \mathrm{otherwise},
\end{cases}
\label{eq:indicator_func}\end{equation}and $C$ is a nonempty, closed and convex set, the proximal operator of $f(x)$ exists and is unique, and it is given by the orthogonal projection operator onto the same set, denoted as $P_C(x)$, so
\begin{equation}
\prox_f(x) = P_C(x) .
\label{eq:prox_indicator}\end{equation}
\end{proposition}

\begin{proposition}[Regularization]
Given $f(x) = \phi(x)+(\rho/2)\|x-a\|_2^2$, the proximal operator of $f(x)$ is given as 
\begin{equation}
\prox_{\lambda f}(v) =  \prox_{\tilde{\lambda} \phi}( (\tilde{\lambda}/\lambda) v + (\rho\tilde{\lambda})a),
\label{eq:prox_reg}\end{equation} with $\tilde{\lambda}= \lambda/(1+\lambda\rho)$.
\end{proposition}

\begin{proposition}[ Log-Sum-Exp function] Given $f(x)=  \log(1+e^{1-x}) $, the proximal operator of $f$, $\prox_{ \lambda f}(x)$, is given as the solution of the following equation with respect to $u$
\begin{equation}
\frac{-e^{1-u}}{1+e^{1-u}}+\frac{u-x}{\lambda}=0.
\label{eq:prox_LSE_final}\end{equation}
\end{proposition}
Note that the proximal operator is given by the minimization
\begin{equation}
\prox_{ \lambda f}(x) = \argmin_u\Big(\log(1+e^{1-u})+\frac{1}{2\lambda}\|u-x\|^2_2\Big).
\label{eq:prox_LSE_og}\end{equation}
By taking the derivative of~\eqref{eq:prox_LSE_og} with respect to $u$ and equating to $0$, we obtain \eqref{eq:prox_LSE_final}.
\subsection{Proximal operator of $f$}
\label{sec:prox_f}
\begin{proposition}
Given $f(x)$ defined as $f(x) = i_{\{u: \textbf{1}^T u = 0\}}(x)$, where $i_S(u)$ is an indicator function of set $S$, as defined in \eqref{eq:indicator_func}, the proximal operator of $f(x)$ is given as
\begin{equation}
\prox_f(x) = x - \bar{\textbf{x}},
\end{equation}
where $\bar{x}$ is the mean of $x$.
\end{proposition}
By \eqref{eq:prox_indicator}, the proximal operator of an indicator function is given as 
\begin{equation*}
\prox_f(x) = P_S(x).
\end{equation*}
In this case $S = \{u: \textbf{1}^T u = 0\}$, so this is an affine set. It is known that the orthogonal projection onto an affine set $C_1 = \{u: Au = b\}$ is given as
\begin{equation*}
P_{C_1}(x) = x - A^T(AA^T)^{-1}(Ax -b) .
\end{equation*}
Taking $A = \textbf{1}^T$ and $b=0$, we obtain
\begin{equation*}
P_{S}(x) = x - \textbf{1} \Big[ \frac{\textbf{1}^Tx}{M}\Big],
\end{equation*}
where we note that the second term is the mean of x, denoted as $\bar{x}$. Consequently, we have
\begin{equation*}
\prox_f(x) = x - \bar{\textbf{x}}.
\end{equation*}

Now, taking $g(x) = f(x) + \gamma \|x\|_2^2$ and by \eqref{eq:prox_reg}, the proximal operator of $g(x)$ is given as
\begin{equation*}
\prox_g(x) = \prox_f\Bigg( \frac{1}{1+2\gamma} x\Bigg) .
\end{equation*}

\subsection{Proximal operator of $g^*(x)$}
\label{sec:prox_g}
\begin{proposition}
Given $g(x)$ defined as $g(x) =  \sum_{n=1}^N w_n \textrm{LSE}(1-e_n^Tx)$, where $\textrm{LSE}(t) = \log(1+e^t)$, the proximal operator of the conjugate function $g^*(x)$ is given as
\begin{equation}
\prox_{\tau g^*}(x) = x -  \tau   ( \prox_{\tilde{w}_n \tilde{g}_n}(\tilde{x}_n) )_{n=1}^N,
\end{equation}where  $\tilde{x}_n = x_n/\tau$, $\tilde{w}_n = w_n/\tau$,  $\tilde{g}_n (y_n) = \textrm{LSE}(1-y_n) $, and $\textrm{prox}_{\tilde{w}_n \tilde{g}_n}(\tilde{x}_n)$ is the solution of 
\begin{equation*}
\frac{-e^{1-u}}{1+e^{1-u}}+\frac{u-\tilde{x}_n}{\tilde{w}_n}=0\end{equation*}
with respect to $u$.\end{proposition}

By Moreau decomposition~(see \cite{book:proximal}, for instance), we have that 
\begin{equation*}
\prox_{\tau g^*}(x) = x - \tau  \prox_{\tau^{-1} g}(x/\tau).
\end{equation*} and since $g$ is separable we have  
\begin{equation*}
\prox_{\tau g^*}(x) = x - \tau   ( \prox_{\tau^{-1} g_n}(x_n/ \tau) )_{n=1}^N,
\end{equation*}
where $g_n (y_n) =  w_n \textrm{LSE}(1-y_n) $. Then, for each component, we have that
\begin{equation*}
\prox_{\tau^{-1} g_n}(x_n/\tau) = \prox_{w_n/ \tau \tilde{g}_n}(x_n/\tau),
\end{equation*}
where $\tilde{g}_n (y_n) =   \textrm{LSE}(1-y_n) $ and $w_n>0$. Then, defining  $x_n/\tau$ as  $\tilde{x}_n$ and $w_n/\tau$ as $\tilde{w}_n$, we can obtain 
\begin{equation*}
\prox_{\tilde{w}_n \tilde{g}_n}(\tilde{x}_n)
\end{equation*} from the solution of Equation~\ref{eq:prox_LSE_final}.Therefore, we have that 
\begin{equation*}
\prox_{\tau g^*}(x) = x -  \tau   ( \prox_{\tilde{w}_n \tilde{g}_n}(\tilde{x}_n) )_{n=1}^N.
\end{equation*}

\section{Further insight into weights}\label{app:weights_plot}
Visual support for the explanation of the correspondence between $w$ and $\omega$ can be found in Figure~\ref{fig:weights_omega}.
\begin{figure}[ht]
	\centering
	\includegraphics[width=\linewidth]{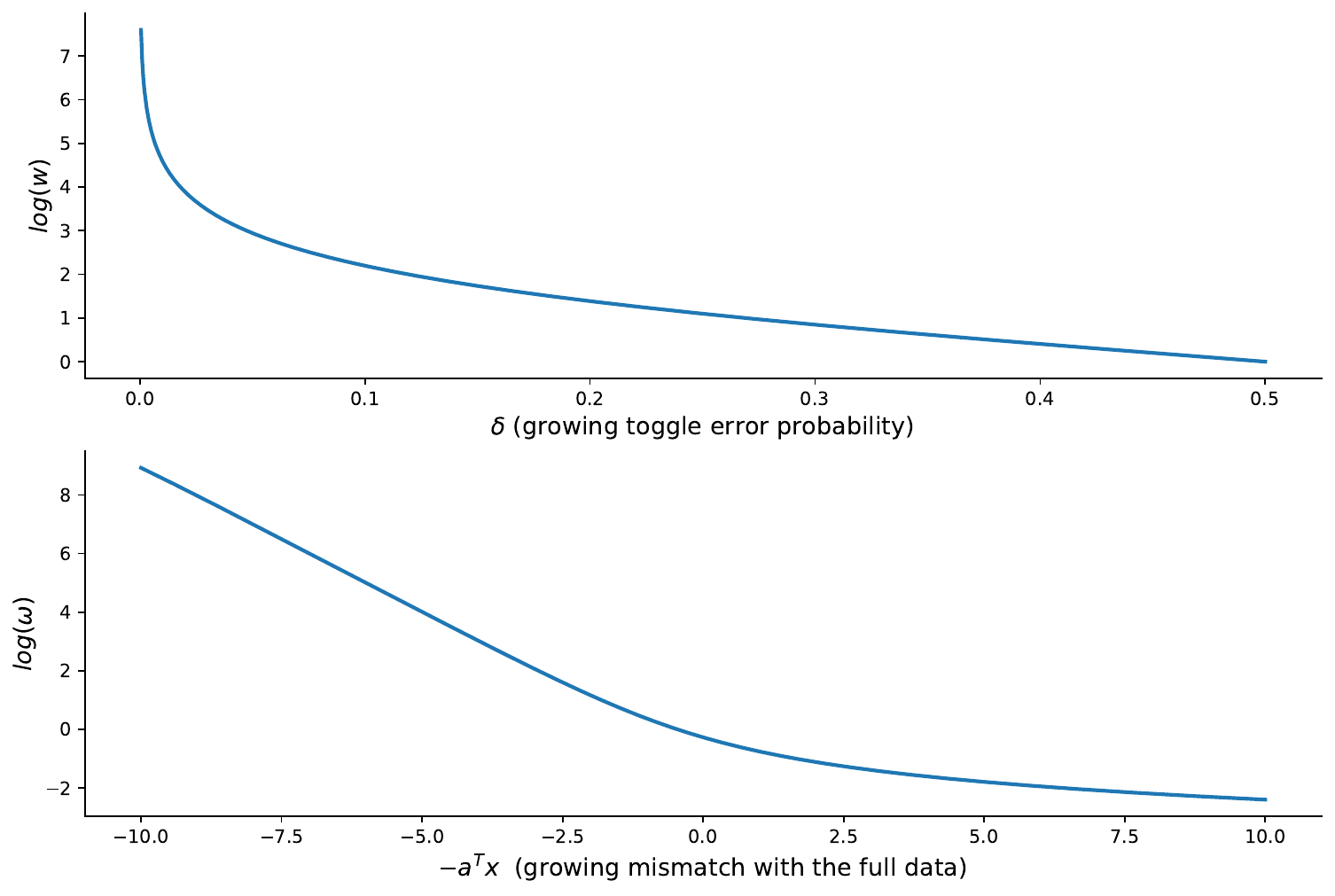}
	\caption{Comparison between variation of $w$ and $\omega$. On the top we vary $\delta$, the noise, between $0$ and $0.5$, depicting the correspondent evolution of $w$. On the bottom we vary the different between two items according to the proposed ranking, with the corresponding $\omega$, according to \eqref{eq:omega}.} \label{fig:weights_omega}
\end{figure}

\section{Experiments}\label{appendix:additional_exp}

\subsection{Additional details on settings}\label{appendix:exp_setting}

\paragraph{Data generation} We generate simulated data with $m$ items, where the ranking corresponds to a random permutation of integers between $1$ and $m$. Then, we randomly select $n$ pairs, with the possibility of selecting repeated pairs. For each pair we take its true label $y_n^{*}$ and create a noisy one by multiplying toggle noise with probability $\delta_n$. For each simulation, we run 50 Monte Carlo trials. When taking the largest connected or strongly connected components we ensure that $m$ is at least $0.8$ of the experiment value.

\paragraph{PD-Rank} For the step-sizes in PDHG we take the maximum values that ensures convergence $\tau = \sigma = \sqrt{ \frac{1}{\|A\|^2_S}}$, $\lambda_n = 1.9$. Parameter $\epsilon$ was set to $ \expnumber{1}{-2}$, it should not affect the performance as long as it is small. The remaining parameters were tested for $m=30$ and $n=1$ standard trial, with $\delta=0.1$. Performance was measured with labels $y$ accuracy. The regularization parameter $\gamma = \expnumber{1}{-4}$ was tested for $[\expnumber{1}{-5},\expnumber{1}{-4},\expnumber{1}{-3},\expnumber{1}{-2},\expnumber{1}{-1}]$ and selected as the largest value without decrease in performance, since computation time decreases with $\gamma$. For the stopping criteria please refer to~\ref{subsec:stop}. 

\paragraph{Benchmark methods} We implemented HodgeRank \cite{article:HodgeRank} with the linear model. MS \cite{article:Minimax} uses half of the available comparisons to estimate the level of noise, which explains its low accuracy overall. The authors assume in the paper that there are $2N$ observations available, but this would not be a fair setting with respect to the remaining methods.  When applicable, the parameters used for the benchmark methods were the ones suggested by the authors.

\paragraph{Simulations} All experiments were run on a 2.60GHz Intel(R) Core(TM) i7-6700HQ CPU with 4 cores and 16GB RAM.

\subsection{Additional Results for Simulated Data}\label{appendix:exp_sim}
\subsubsection{Stopping criteria}\label{subsec:stop}

We compare our algorithm with the solution obtained from a standard solver -- CVXPY \cite{code:cvx_py} -- iteratively applied to Problem~\eqref{eq:cvxSubproblem}.
The stopping criterion for the PDHG algorithm (i.e. the inner cycle) is taken as the difference in cost between consecutive iterations, as a fraction of the first, smaller than $\epsilon_{in}$. For the outer cycle, we take either the stabilization of the weight values or the score values, by taking the difference between the current and previous value, as a faction of the previous. We call this parameter $\epsilon_{\textrm{out}}$ and set it to $ \expnumber{1}{-2} $, except for the large-scale experiments ($m>500$), when it is set to $ \expnumber{1}{-1} $. In Figure~\ref{fig:epsilon_in}, we can see that $\epsilon_{in}$ has a large impact on the computation time of our algorithm, but it can be increased up to $\expnumber{1}{-2}$ with minimal decrease in the accuracy and to $\expnumber{1}{-3}$ with no decrease. We have sett $\epsilon_{\textrm{in}} = \expnumber{1}{-2} $ for $m \leq 500$ and $\epsilon_{\textrm{in}} = \expnumber{1}{-3} $, otherwise. Nonetheless, it is evident the advantage we obtain in using our algorithm compared to a naive implementation.

\begin{figure}[ht]
	\centering
	\includegraphics[width=	\columnwidth]{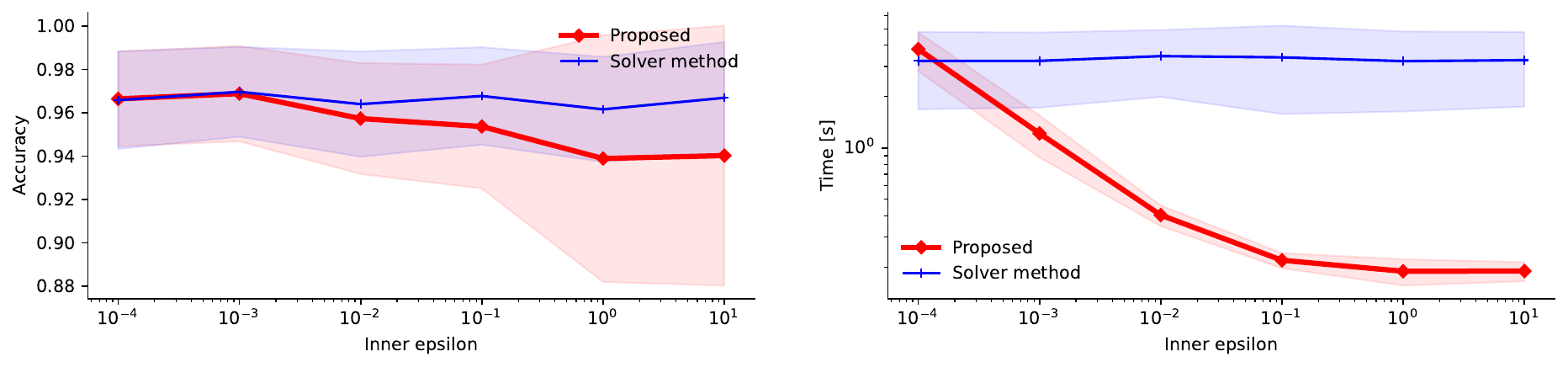}
	\caption{Variation of label accuracy and total computation time with $\epsilon_{in}$. The left figure depicts the accuracy of predicted labels for each observation with respect to the true labels. The shaded area corresponds to $1$ standard deviation. }\label{fig:epsilon_in}
\end{figure}

\subsubsection{Outliers}
We consider a setting with 5 items with their scores corresponding to their index, i.e. item number $4$ has a score of $4$ and is ranked first. We set $n=30$ standard trials and $\delta=0.1$ for all edges except one. In the edge between node $0$ and $3$ we introduce an outlier by setting $\delta$ to $0.9$. Since $n$ is considerably high, in general all the comparisons will represent the real ordering of items with $10\%$ of opposite comparisons, except for the outlier edge (see Figure~\ref{fig:outliers} for the corresponding graph). PD-Rank is able to handle this scenario much better than all benchmark methods, retrieving the correct ranking $250$ out of $300$ times (Figure~\ref{fig:outliers_res})

\begin{figure}[ht]
	\centering
	\includegraphics[scale=0.25]{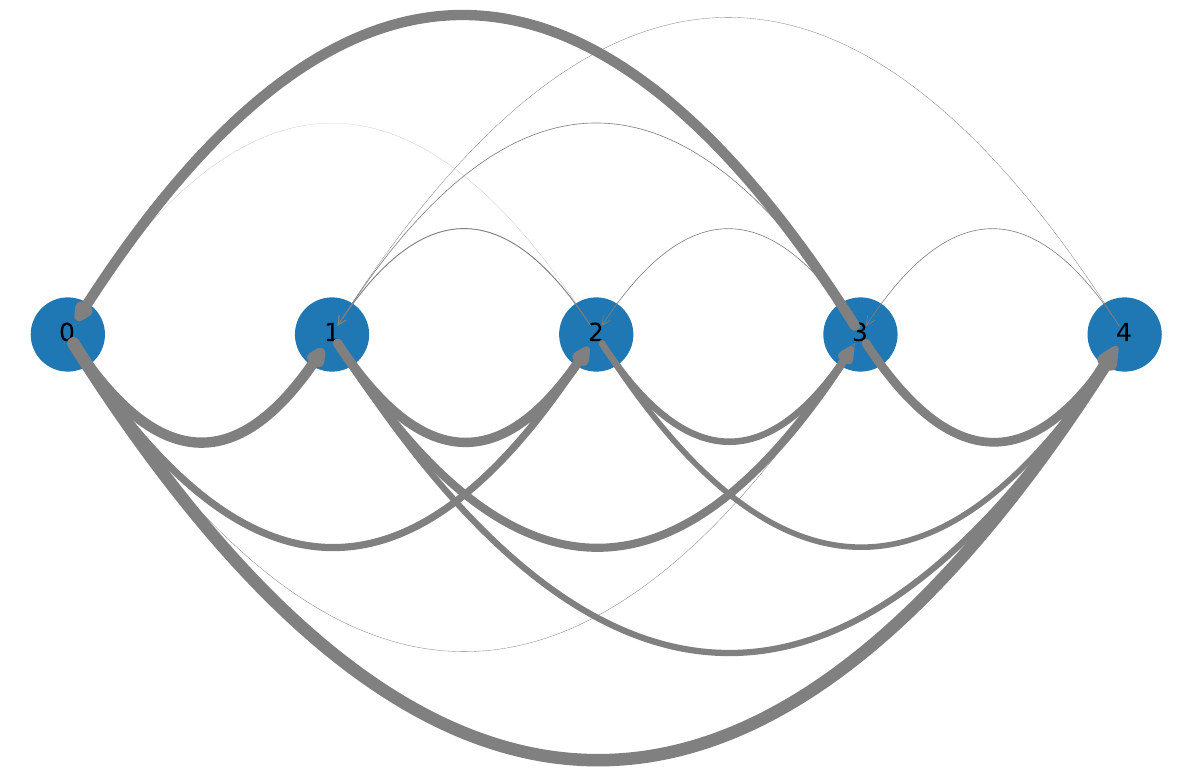}
	\caption{Example graph for observations in outliers experiment. Each node corresponds to a different item and each edge to a pairwise comparison between the corresponding nodes. The width of each edge is proportional to the number of observations for the respective pair such that thicker edges translate to more observations. Pair $0-3$ is an outlier, subjected to $\delta=0.9$, while the remaining ones are subjected to $\delta=0.1$.} \label{fig:outliers}
\end{figure}

\begin{figure}[h]
\centering
\subfloat[]{
        \begin{tabular}[b]{l|lllll}
                  & \multicolumn{5}{l}{Ranking} \\ \hline
        GT        & 4   & 3   & 2   & 1   & 0   \\
        PD-Rank   & 4   & 3   & 2   & 1   & 0   \\
        Borda     & 4    &   2  & 3    & 0    &   1  \\
        HodgeRank & 4    &    2 &  3   &  0   &  1   \\
        I-LSR     & 4    &     2&   0  &   1  & 3    \\
        ASR       &  4   &   2  &    0 &    1 &3     \\
        RC        &   4  &    2 &    3 &     0&     1\\
        SVM       &    0 &    1 &     4&     2&    3 \\
        MS        & 4    & 2    & 0    &3     &1    
        \end{tabular}
    \label{fig:sub1}
}

\subfloat[]{
\includegraphics[width=0.4\linewidth]{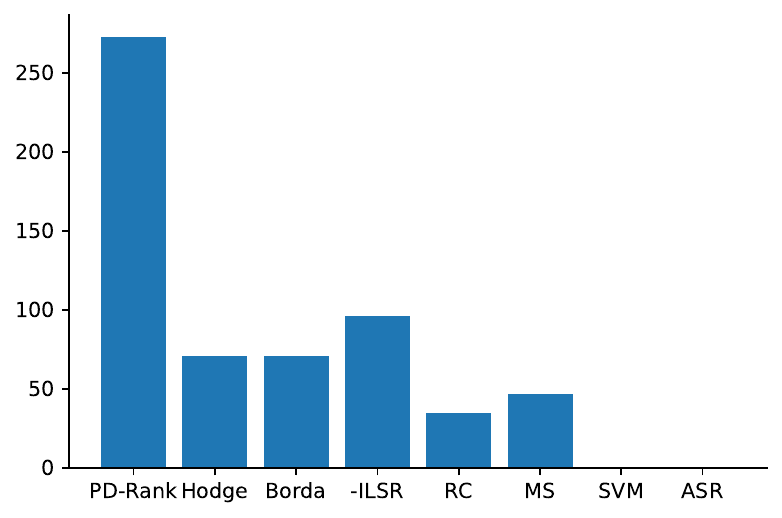}
\label{fig:sub2}
}
\caption{Results for outlier experiments corresponding to the graph in \ref{fig:outliers}. On the left we have an example of ranking where PD-Rank is able to uncover the full correct rank and the other methods fail. On the right the number of times each method is able to get the exact ranking out of 300 trials.}
\label{fig:outliers_res}
\end{figure}

\subsubsection{BT model with connected graph}\label{app:BT_connected}
We present the results for the same setting as Figure~\ref{fig:BTModel} of the main mauscript, but for the largest connected component instead of strongly connected. We can observe that they are similar to the previous ones, except some methods can no longer be applied. The results are presented in Figure~\ref{fig:BTModel_cnn}.
\begin{figure}[H]
\centering
	\begin{subfigure}{.45\textwidth}
	\centering
		\includegraphics[width=\textwidth]{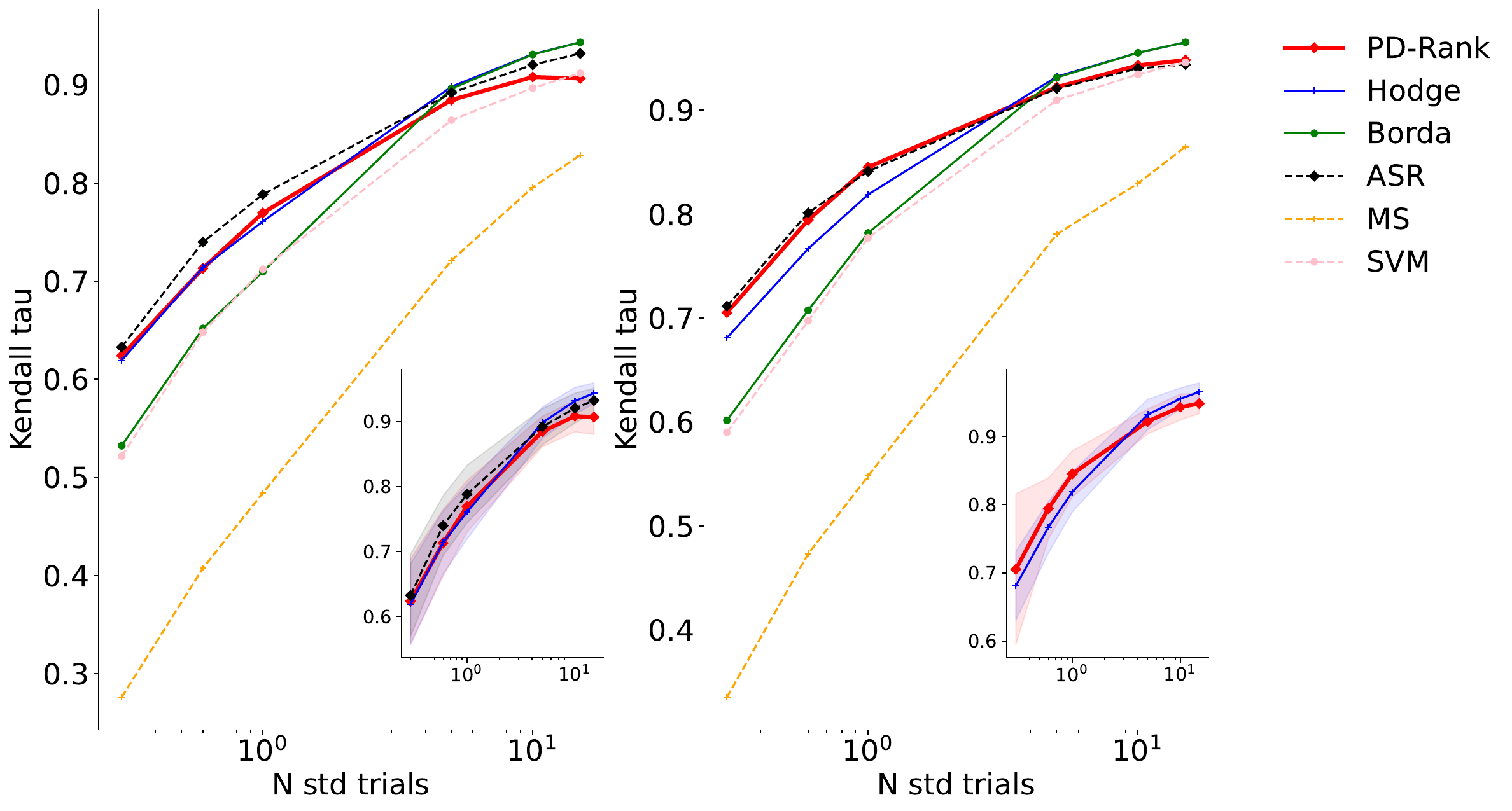}
		\caption{Kendall tau}
	\end{subfigure}	
	\begin{subfigure}{.45\textwidth}
		\centering
		\includegraphics[width=\textwidth]{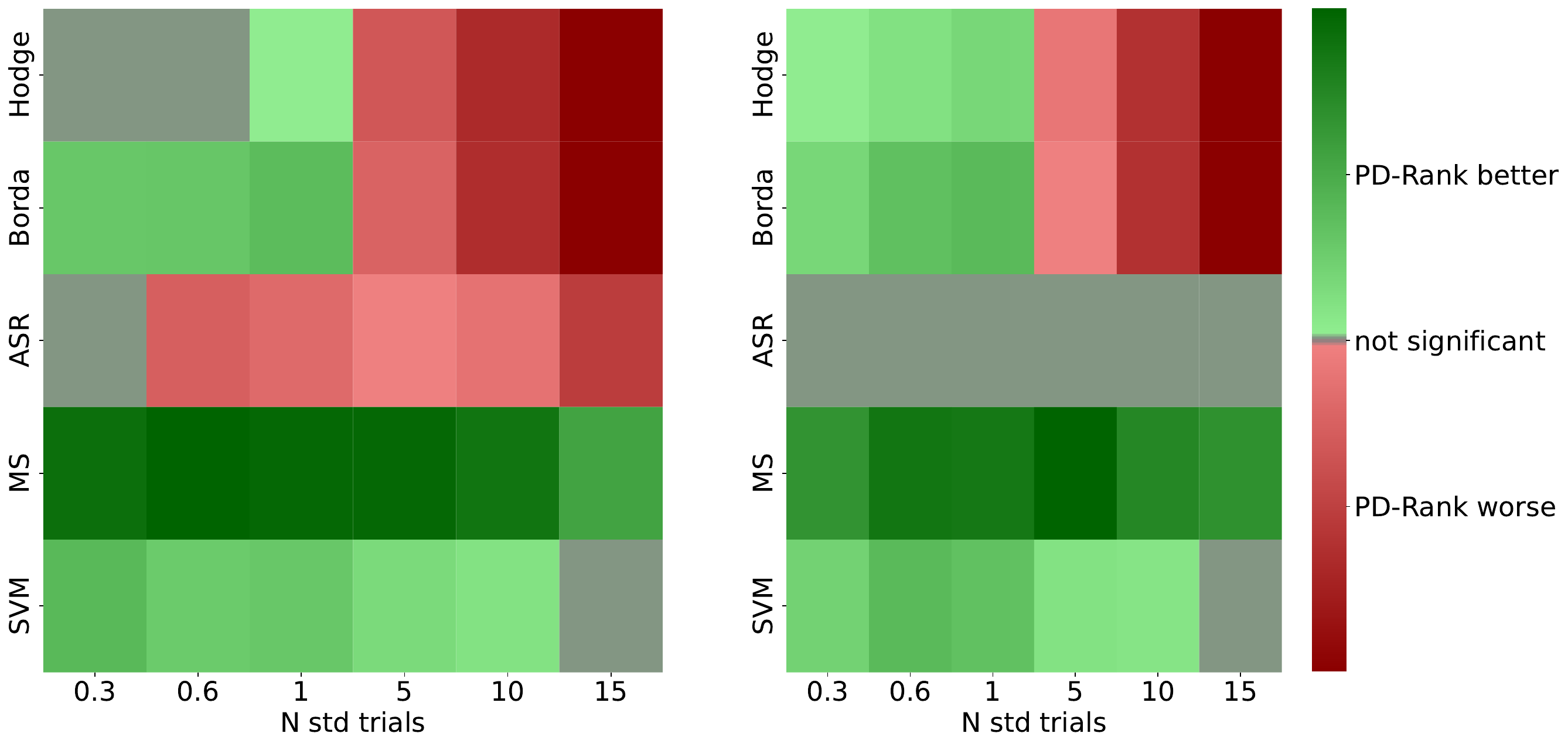}
		\caption{Results of paired t-test between Kendall scores of PD-Rank and other methods}
	\end{subfigure}	
	\caption{Simulated data according to the BT model for connected graphs, on the left scores in range $[0,5]$ and on the right $[0,10]$. For the heatmap, the color is proportional to the statistic value, being green if PD-Rank performs better and red, otherwise; if the p-value is smaller then $0.05$ the color is set to grey. }
	\label{fig:BTModel_cnn}
\end{figure}

\subsubsection{Additional Metrics} In the main paper, we present Kendall's Tau as the primary metric for evaluating our method. Here, we extend our evaluation by including additional metrics: Spearman's rank correlation coefficient and the number of items correctly retrieved in the top-K. The top-K metric is defined as the size of the intersection between the predicted top-K items, $X^k$, and the true top-K items, $Y^k$, i.e., $\big| X^k \cap Y^k\big|$.

\paragraph{Independent noise model (Figure~\ref{fig:ourModel_AddMetrics}).} We consider the same setting as in Figure~\ref{fig:ourModel}, for a model with noise independent of item scores. We present additional metrics with paired t-tests between the scores of PD-Rank and the other methods. PD-Rank generally outperforms the other methods. An exception is SVM, which demonstrates a superior performance when a large number of comparisons is available (15 standard trials) in both Kendall's and Spearman's metrics, for strongly connected graphs. For lower values of $n$, there is not enough information to retrieve a meaningful rank and the Spearman's score are not statistically relevant, while for the remaining metrics PD-Rank surpasses other methods. For larger values of $n$, PD-Rank is less competitive in retrieving the top-K items.

\begin{figure}[t]
	\begin{subfigure}{\columnwidth}
		\includegraphics[width=\columnwidth]{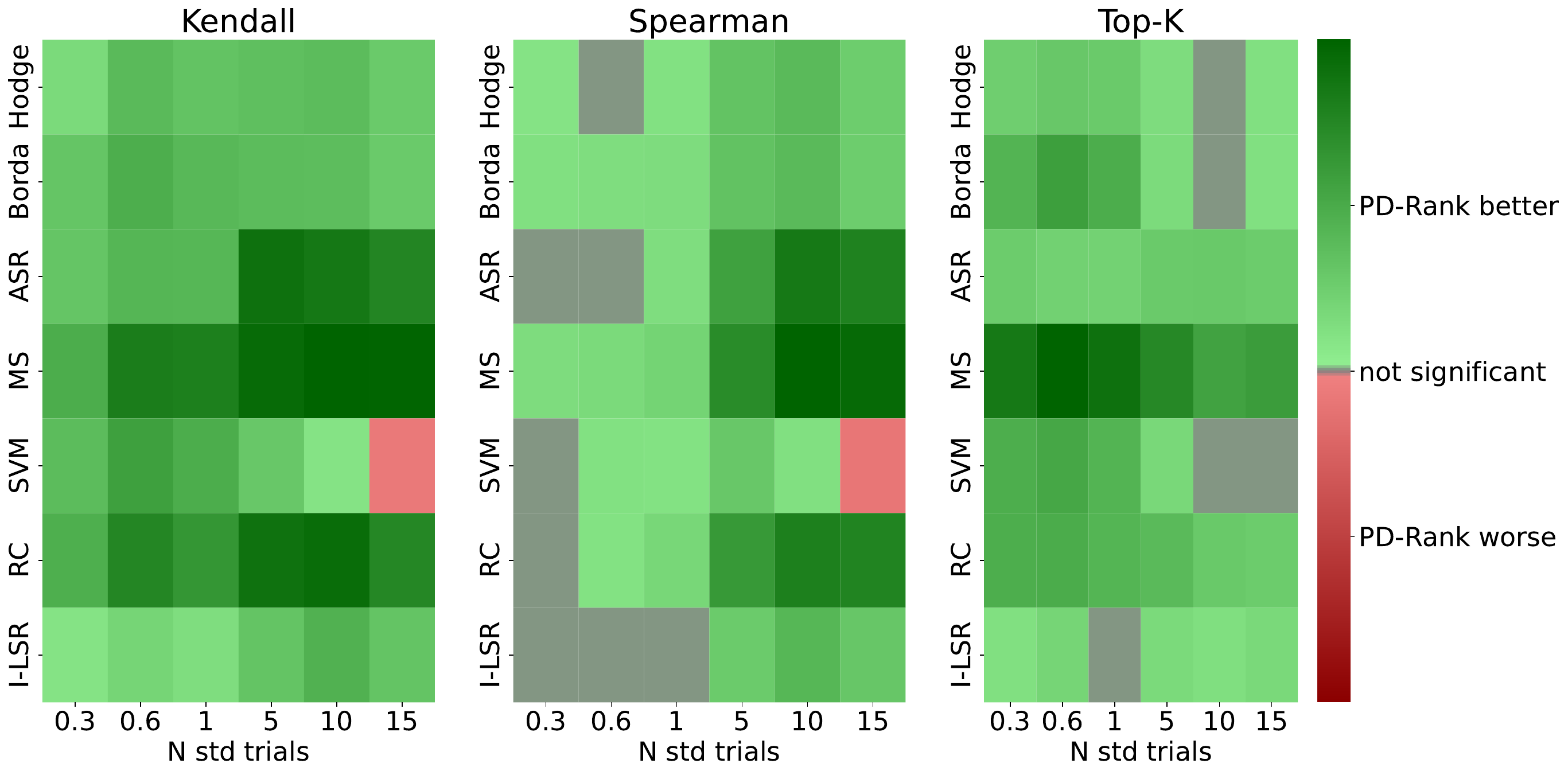}
		\caption{Stronlgy connected graphs.}\label{fig:spe}
	\end{subfigure}
	\begin{subfigure}{\columnwidth}
		\includegraphics[width=\columnwidth]{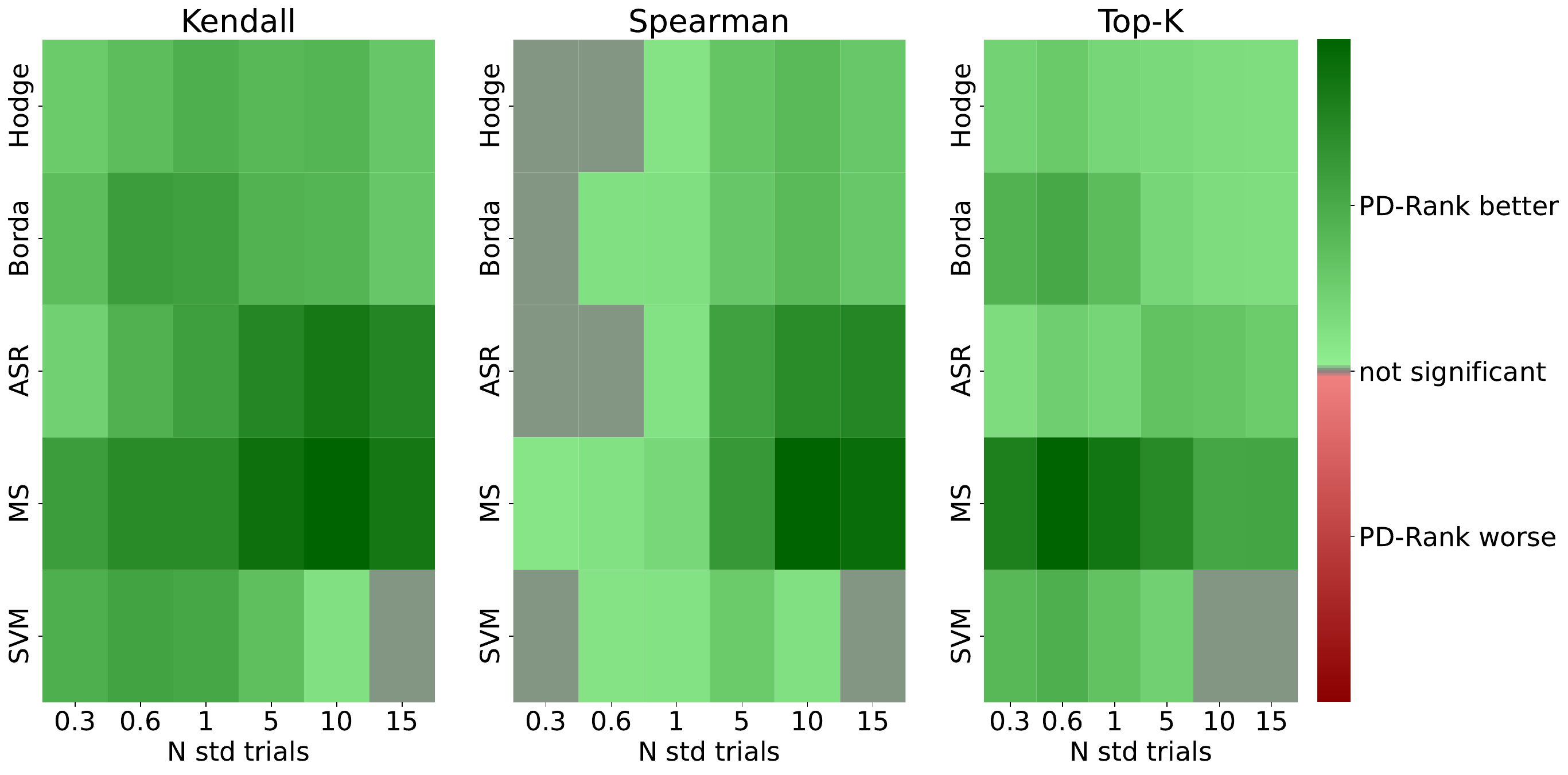}
		\caption{Connected graphs.}\label{fig:d}
	\end{subfigure}	
	\caption{Simulated data with noise independent of score value, for the largestst connected and strongly connected graphs. Variation with the number of comparisons. The color is proportional to the statistic value, being green if PD-Rank performs better and red, otherwise; if the p-value is smaller than $0.05$ the color is set to grey. We observe the Kendall tau, Spearman coefficient and Top-k retrieval.}
	\label{fig:ourModel_AddMetrics}
\end{figure}

\paragraph{BT model (Figure~\ref{fig:BTModel_Extra}).} We consider the same setting as in Figure~\ref{fig:BTModel}.
PD-Rank has a similar or superior performance to the non-parametric models SVM and MS, while Borda (the third non-parametric approach), performs better for larger values of $n$. I-LSR consistently outperforms or matches PD-Rank across all metrics. The remaining parame methods (RC, Hodge Rank, and ASR) are are in general less competitive for lower values of $n$ and superior for higher values. There is no substancial difference across the three metrics.

\begin{figure}[t]
	\begin{subfigure}{\columnwidth}
		\includegraphics[width=\columnwidth]{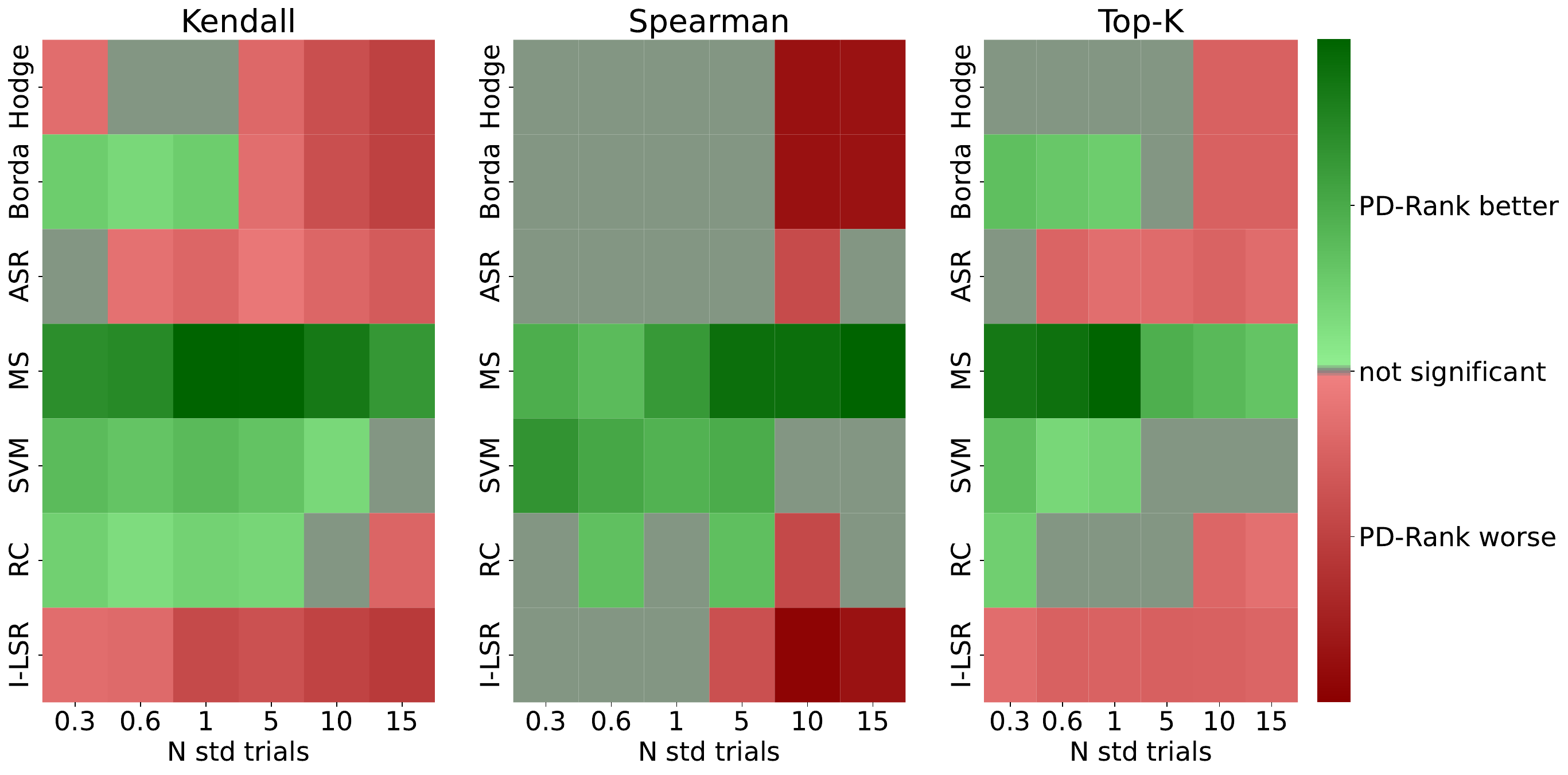}
		\caption{Strongly connected}
	\end{subfigure}	
	\begin{subfigure}{\columnwidth}
		\includegraphics[width=\columnwidth]{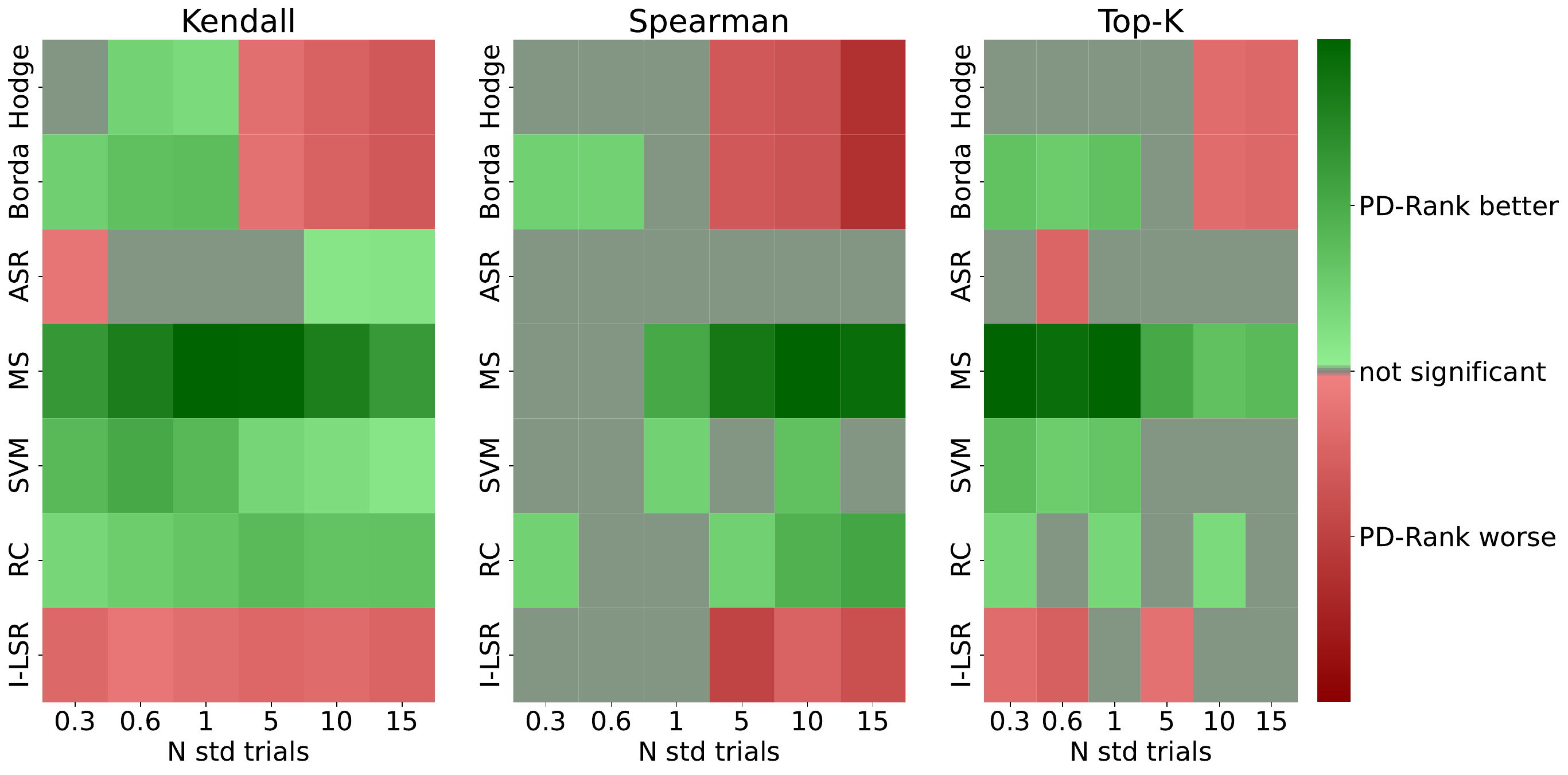}
		\caption{Connected}
	\end{subfigure}	
	\caption{Simulated data according to the BT model for strongly connected and connected graphs, on the left scores in range $[0,5]$ and on the right $[0,10]$. For the heatmap, the color is proportional to the statistic value, being green if PD-Rank performs better and red, otherwise; if the p-value is smaller than $0.05$ the color is set to grey. }
	\label{fig:BTModel_Extra}
\end{figure}

\paragraph{Large scale setting (Figure~\ref{fig:largedata_app}).} We consider the same setting as in Figure~\ref{fig:largedata}. In this case we see a remarkable difference across the metrics. While PD-Rank has a significantly better performance in Kendall's, the opposite in observed for the Top-K retrieval, as expected. Moreover, in terms of Spearman's coefficient, there is no statistical significant difference between PD-Rank and the other methods, except for $m=100$, where PD-Rank is superior. This corresponds to approximately $0.1$ standard trial.

\begin{figure}[t]
	\centering
		\includegraphics[width=	\columnwidth]{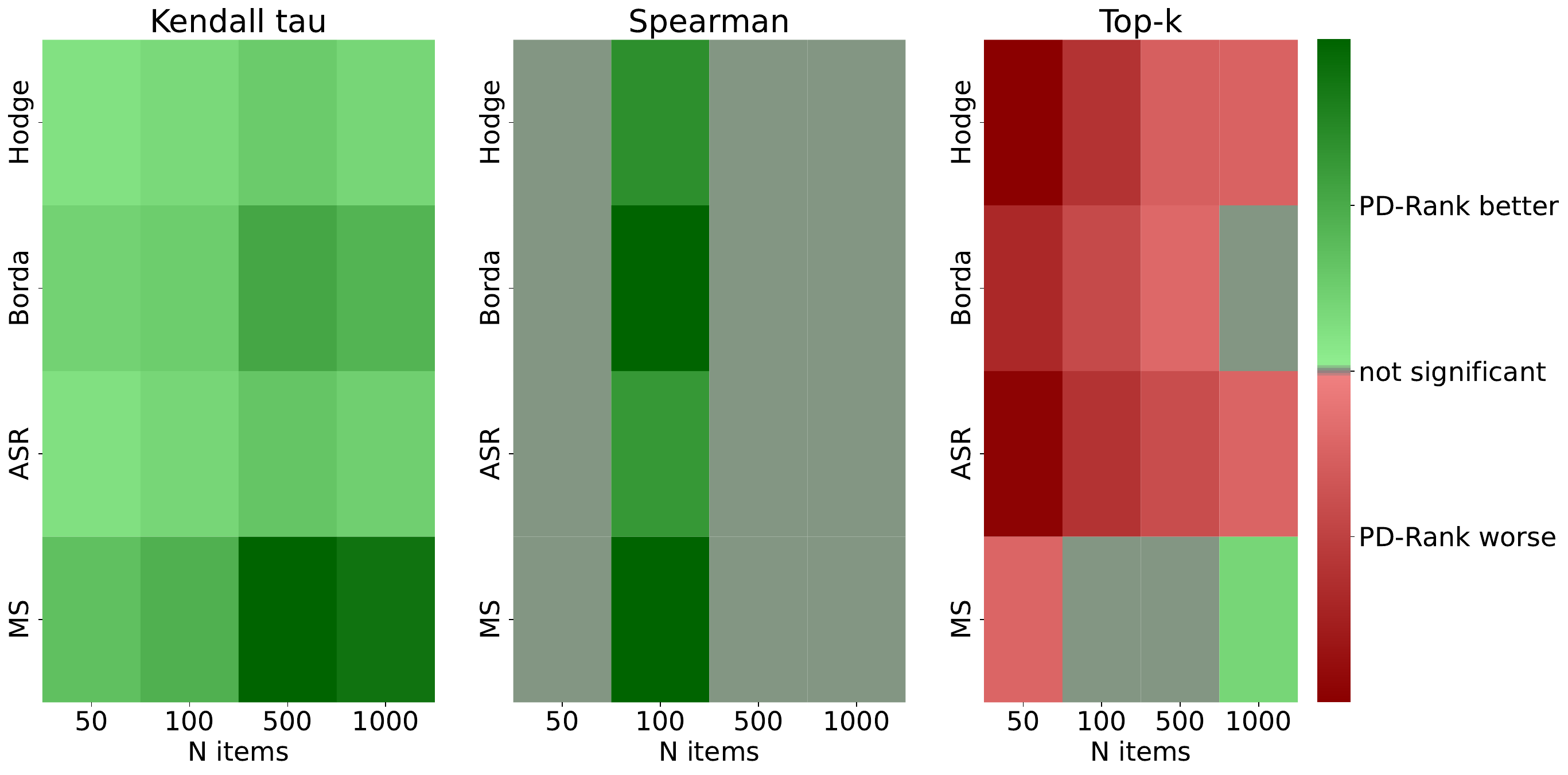}
		\caption{Large data setting, with variation of number of items $m$, with $5$ comparisons per item. The color is proportional to the statistic value, being green if PD-Rank performs better and red, otherwise; if the p-value is smaller than $0.05$ the color is set to grey. We observe the Kendall tau, Spearman coefficient and Top-k retrieval.}
		\label{fig:largedata_app}
	\end{figure}
	
\paragraph{Key points.} These results attest that PD-Rank is particularly suitable for full rank retrieval (and not Top-K retrieval), in a setting of limited pairwise comparisons. It is capable of handling noisy observations and is well-suited for scenarios where noise is independent of item scores.

\subsection{Additional Experiments for Real-world Datasets}\label{appendix:exp_real}

We include some additional references for datasets VQA (Figure~\ref{fig:VQA}) and IQA (Figure~\ref{fig:IQA}), that reinforce the conclusions obtained in the main paper. Recall that we are taking as ground truth (GT) the solution obtained with the Bradley-Terry model for the full dataset and this brings some unfairness to Borda Count and PD-Rank, as they are usually not able to reach Kendall of 1.0, since the GT according to all methods is not the same. We confirm that for higher accuracies (over $0.85$) PD-rank achieves the same Kendall as Hybrid-MST and Hodge Rank with less computation time. The latter two, being active learning methods arrive at a higher Kendall tau with less available comparisons, but they require more time, as they must be run sequentially. Borda Count is always faster than any other candidate by several orders of magnitude, but it requires a much larger data sample and often does not reach high enough values of Kendall tau.

\begin{figure*}[htb]
\subfloat[Reference 1]{%
\includegraphics[width=0.5\linewidth]{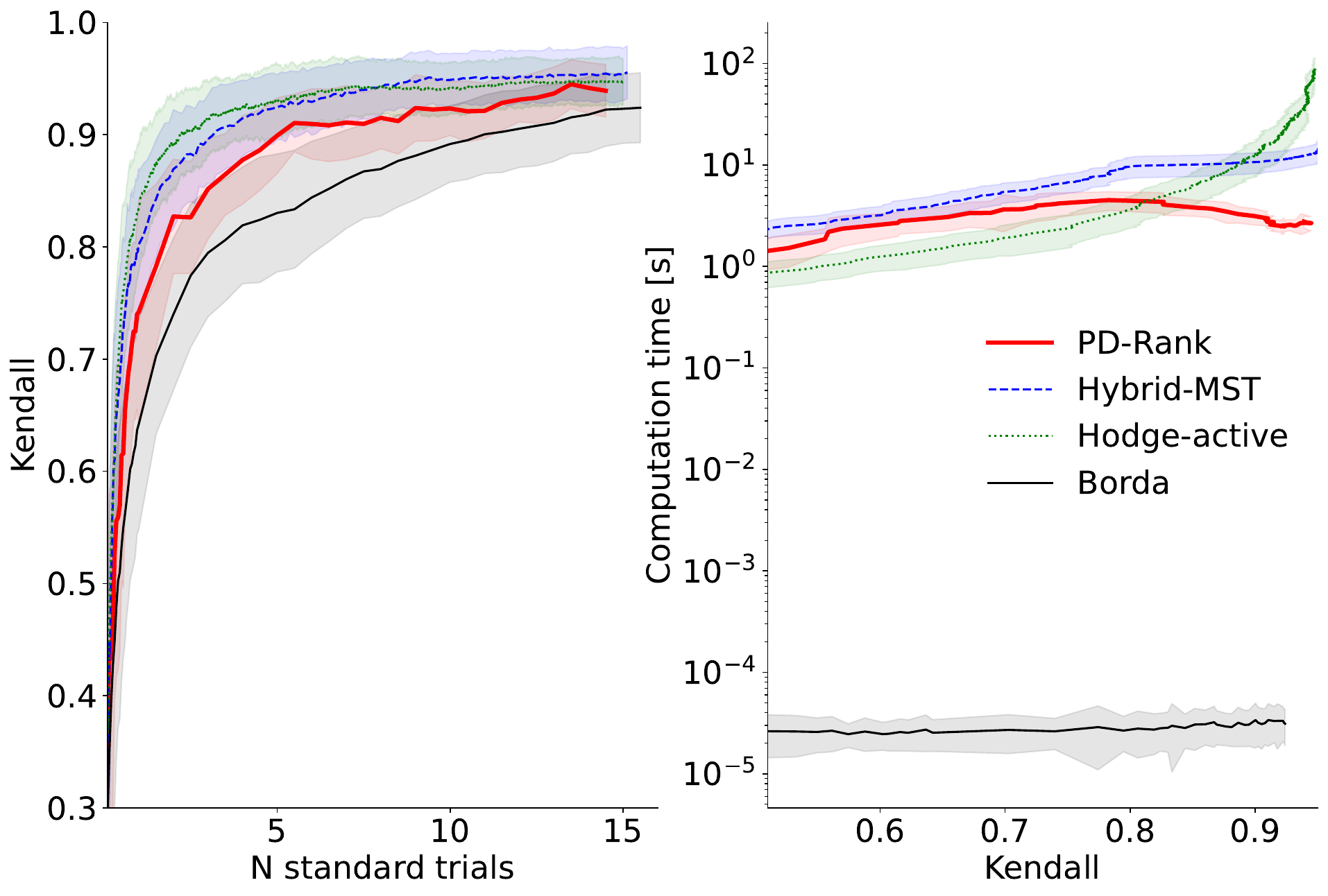}%
}
\subfloat[Reference 2]{%
\includegraphics[width=0.5\linewidth]{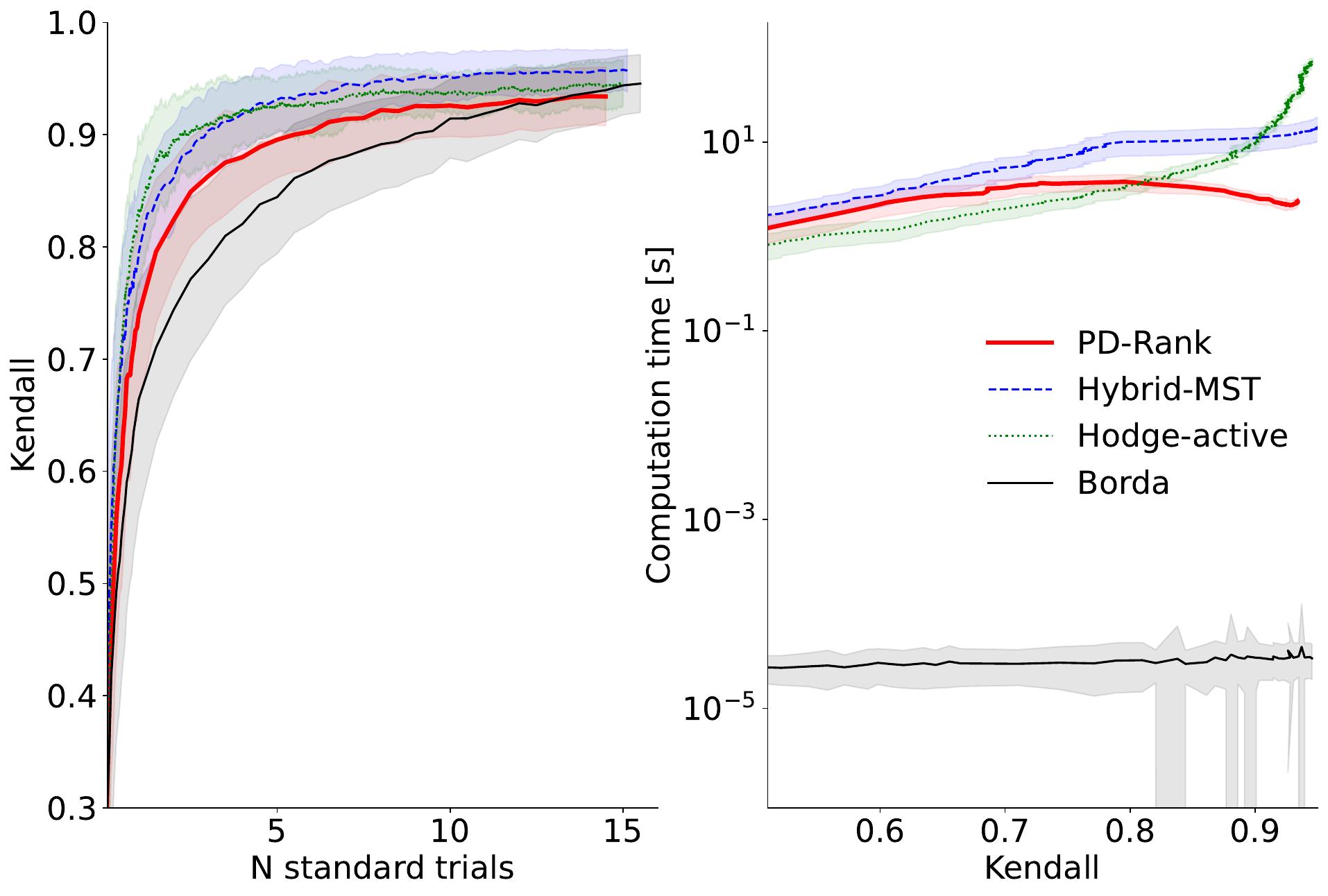}%
}

\subfloat[Reference 3]{%
\includegraphics[width=0.5\linewidth]{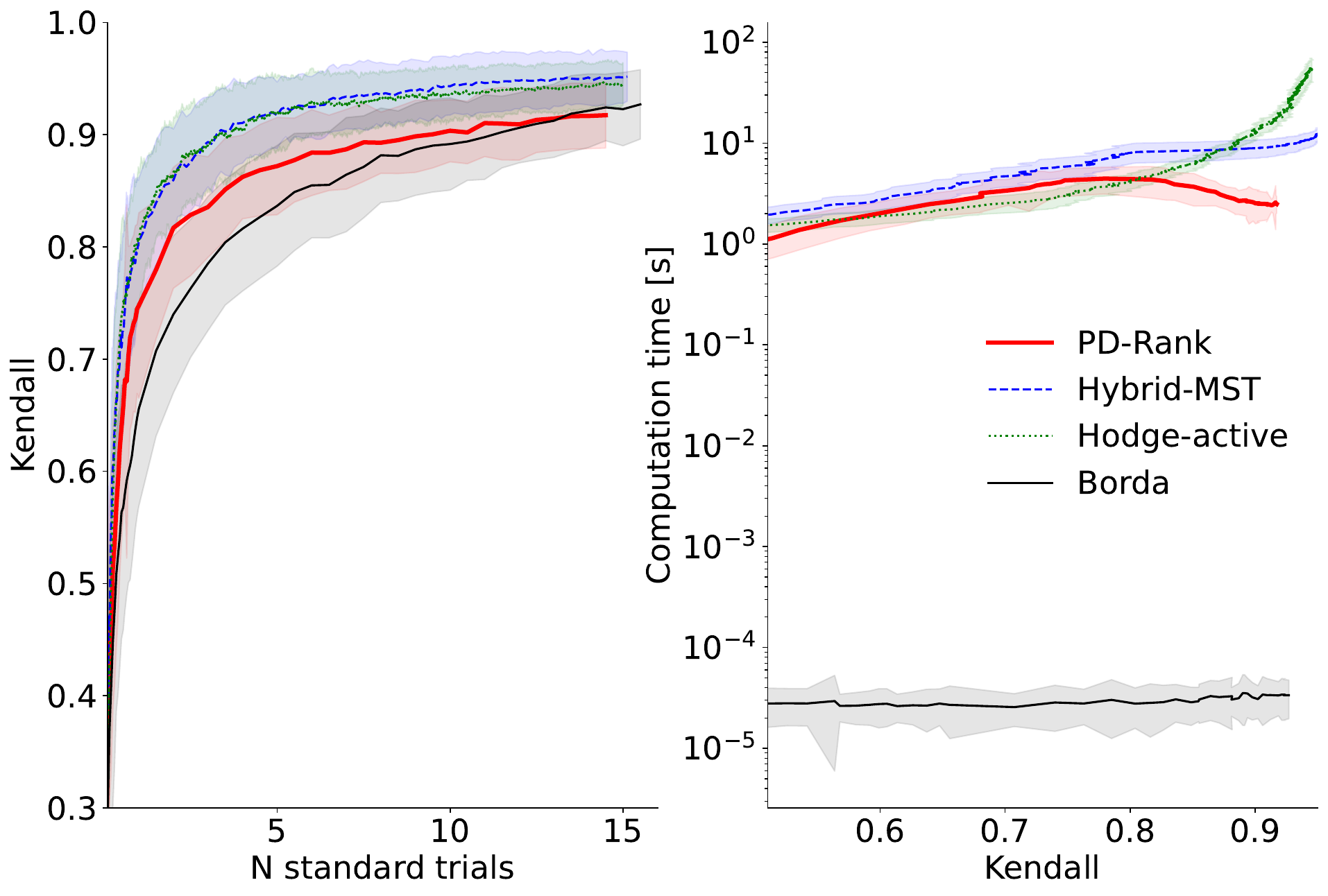}%
}
\subfloat[Reference 4]{%
\includegraphics[width=0.5\linewidth]{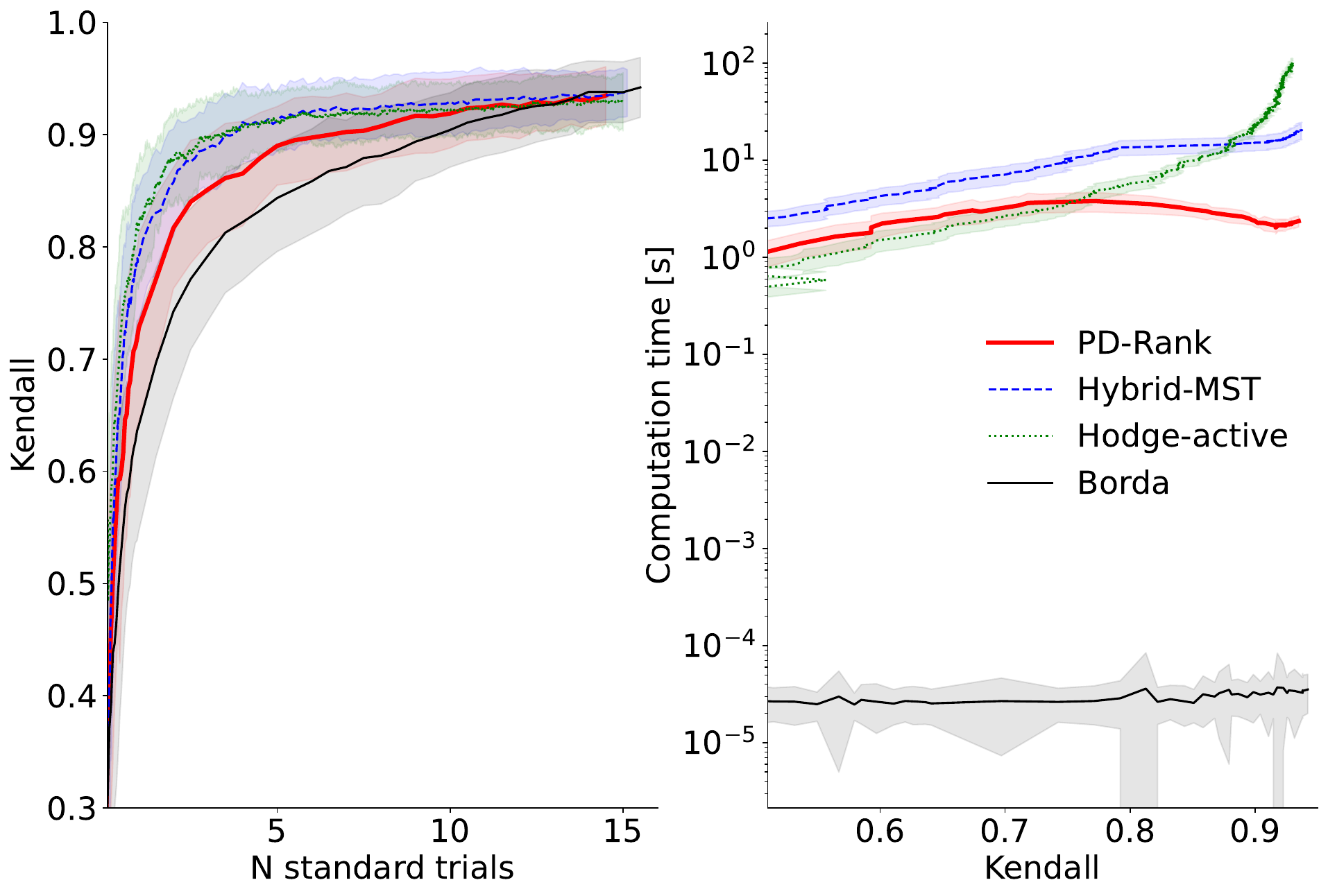}%
}	
\caption{Additional results for simulations with VQA-dataset. For each subfigure, on the left plot, we can see the evolution of ranking accuracy with the observed samples. Observed samples are expressed in terms of standard trials, where \textbf{1 standard trial} is defined as $m(m-1)/2$, for $m$ items. On the right, we show the evolution of computation time with the achieved Kendall coefficient: lines above PD-Rank require more time for the same accuracy and vice-versa.}
\label{fig:VQA}
\end{figure*}

\begin{figure*}[htb]
\subfloat[Reference 1]{%
\includegraphics[width=0.5\linewidth]{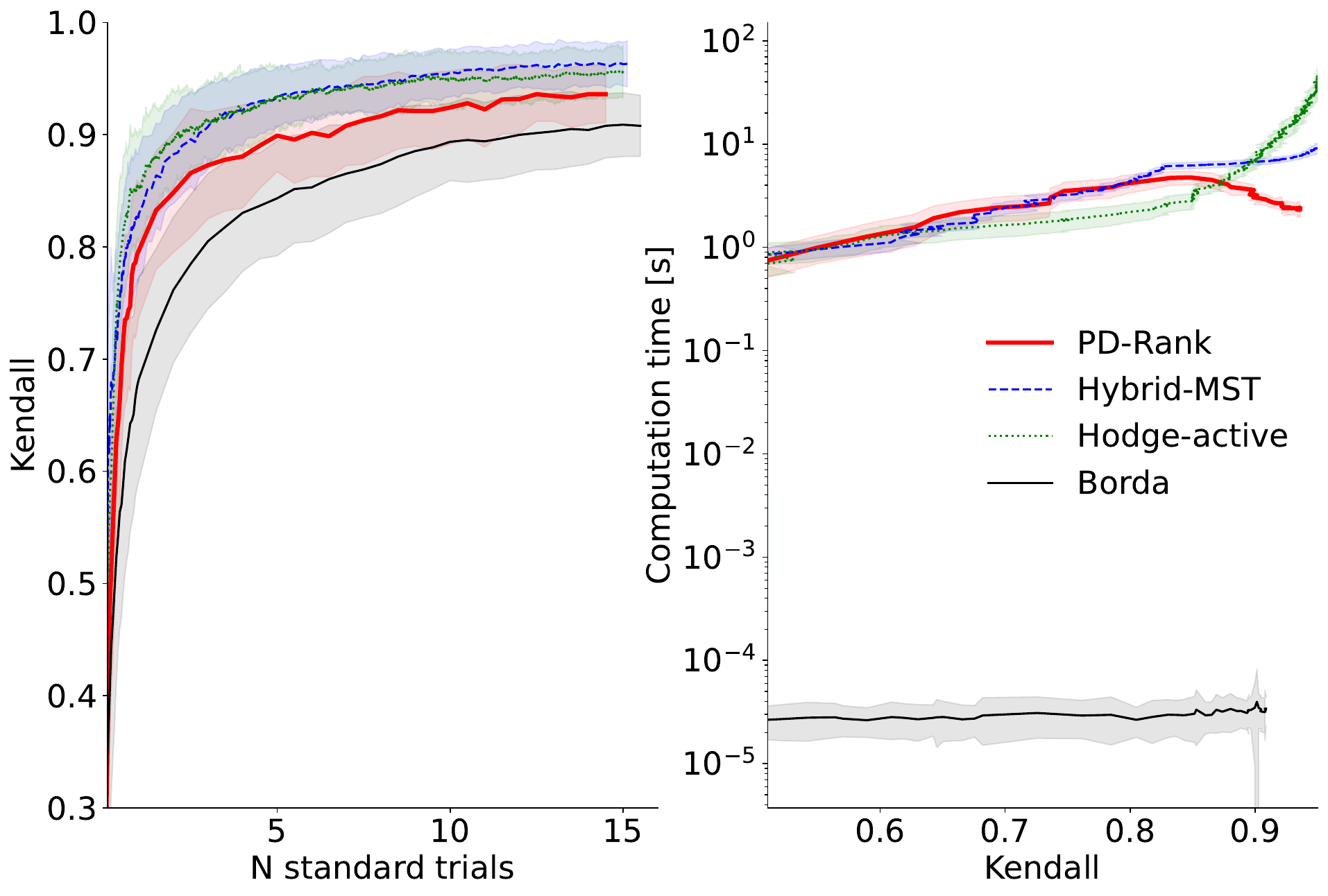}%
}
\subfloat[Reference 2]{%
\includegraphics[width=0.5\linewidth]{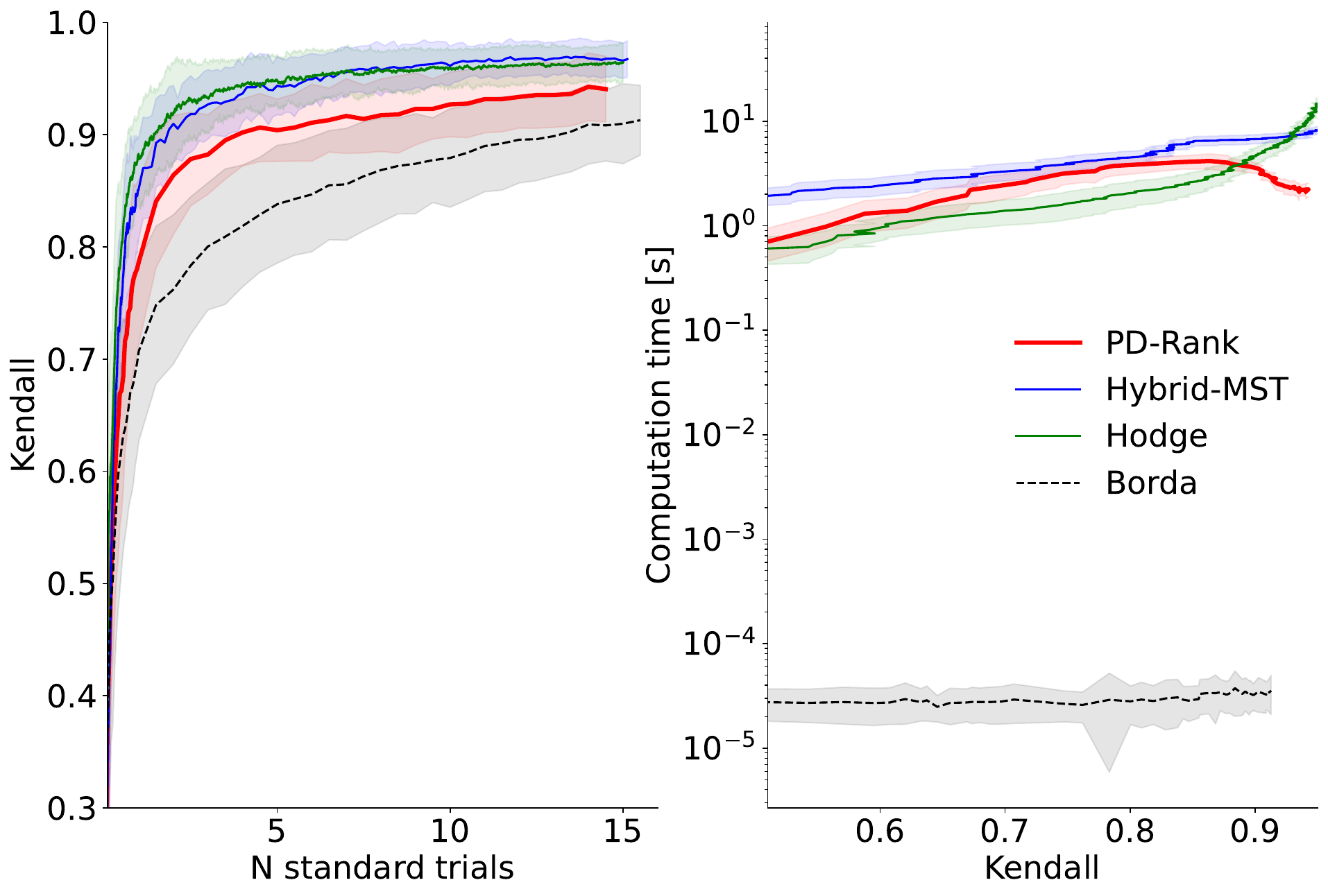}%
}

\subfloat[Reference 3]{%
\includegraphics[width=0.5\linewidth]{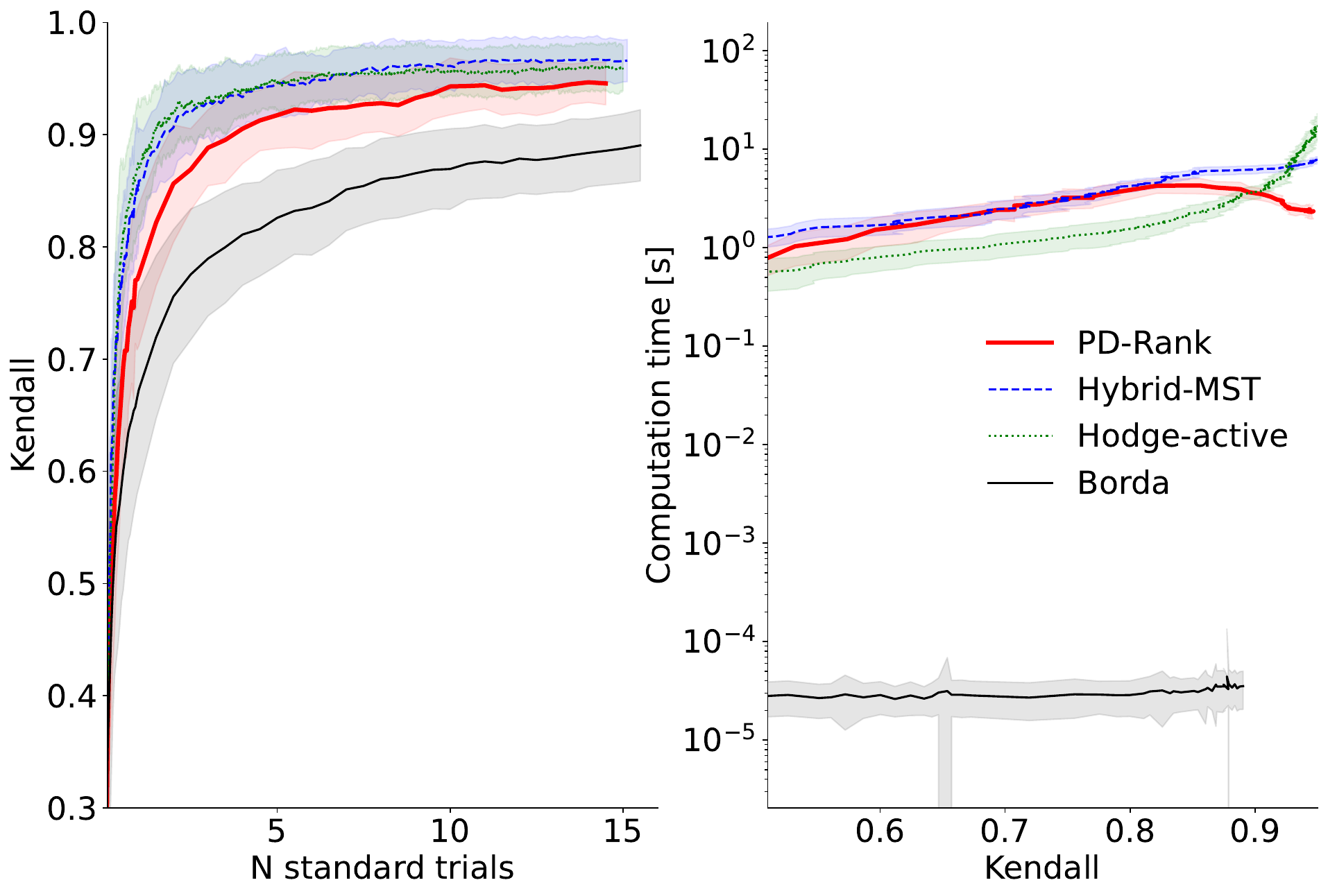}%
}
\subfloat[Reference 4]{%
\includegraphics[width=0.5\linewidth]{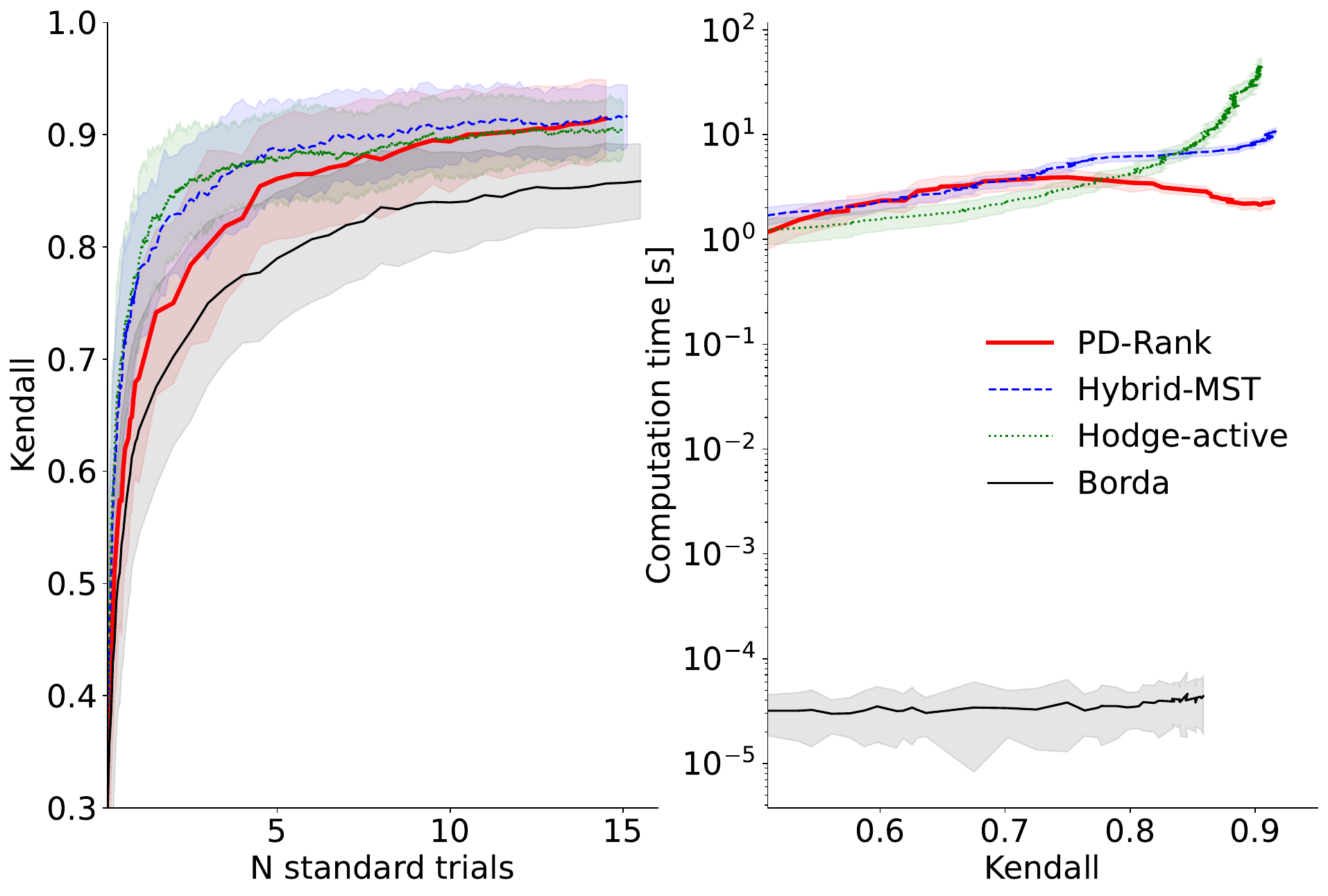}%
}

\caption{Additional results for simulations with IQA-dataset. For each subfigure, on the left plot, we can see the evolution of ranking accuracy with the observed samples. Observed samples are expressed in terms of standard trials, where \textbf{1 standard trial} is defined as $m(m-1)/2$, for $m$ items. On the right, we show the evolution of computation time with the achieved Kendall coefficient: lines above PD-Rank require more time for the same accuracy and vice-versa.}
\label{fig:IQA}
\end{figure*}

\end{document}